\pdfoutput=1
\RequirePackage{silence}
\documentclass[letterpaper, 10 pt, journal, twoside]{IEEEtran}
\usepackage{silence,graphicx,float,dsfont,amsfonts,color}
\usepackage{amsmath,amsfonts, amssymb,commath,algorithm,algpseudocode,comment,booktabs,multirow,makecell,pifont,etoolbox}%
\usepackage[noadjust]{cite}

\usepackage{enumitem}
\usepackage[margin=10pt, font=footnotesize, labelfont=rm, textfont=rm]{subfig}
\usepackage{footnote}
\usepackage[normalem]{ulem}
\makesavenoteenv{algorithm}

\makeatletter
\let\NAT@parse\undefined
\makeatother
\usepackage[hypertexnames=false,hidelinks]{hyperref}
\usepackage{cleveref}
\usepackage[all]{hypcap}

\graphicspath{{Figures/}}
\crefname{equation}{}{}

\DeclareUnicodeCharacter{FF0C}{,}

\let\oldforall\forall
\let\forall\undefined
\DeclareMathOperator{\forall}{\oldforall}

\newtheorem{lemma}{Lemma}
\AtBeginEnvironment{lemma}{\setlength{\parindent}{0pt}}
\AtEndEnvironment{lemma}{\setlength{\parindent}{15pt}}
\newtheorem{theorem}{Theorem}
\AtBeginEnvironment{theorem}{\setlength{\parindent}{0pt}}
\AtEndEnvironment{theorem}{\setlength{\parindent}{15pt}}
\newtheorem{problem}{Problem}

\newtheorem{definition}{Definition}
\AtBeginEnvironment{definition}{\setlength{\parindent}{0pt}}
\AtEndEnvironment{definition}{\setlength{\parindent}{15pt}}
\newtheorem{assumption}{Assumption}
\newtheorem{proof}{Proof}

\renewcommand{\secref}[1]{Section~\ref{#1}}
\renewcommand{\figref}[1]{Fig.~$\ref{#1}$}
\renewcommand{\algref}[1]{Algorithm~$\ref{#1}$}

\newcommand{\lemmaref}[1]{Lemma~$\ref{#1}$}
\newcommand{\linref}[1]{line~$\ref{#1}$}

\renewenvironment{proof}{\begin{IEEEproof}}{\end{IEEEproof}\ignorespacesafterend}

\allowdisplaybreaks[4]

\title{\vspace{-1em}
{\large Accepted to IEEE Transactions on Robotics (TRO)\\[1em]}Double Oracle Algorithm for \\ Game-Theoretic Robot Allocation on Graphs}

\author{Zijian An and Lifeng Zhou*
\thanks{Zijian An and Lifeng Zhou are with the Department of Electrical and Computer Engineering, Drexel University, Philadelphia, PA 19104, USA. Email: \texttt{\small \{za382,lz457\}@drexel.edu}.}
\thanks{This research was sponsored by the Drexel Seaman Endowed Fellowship.}
\thanks{* Corresponding author.}
}

\begin{document}
\bstctlcite{IEEEexample:BSTcontrol}
\maketitle
\begin{abstract}

    We study the problem of game-theoretic robot allocation where two players strategically allocate robots to compete for multiple sites of interest. Robots possess offensive or defensive capabilities to interfere and weaken their opponents to take over a competing site. This problem belongs to the conventional Colonel Blotto Game. Considering the robots' heterogeneous capabilities and environmental factors, we generalize the conventional Blotto game by incorporating heterogeneous robot types and graph constraints that capture the robot transitions between sites. Then we employ the Double Oracle Algorithm (DOA) to solve for the Nash equilibrium of the generalized Blotto game. Particularly, for cyclic-dominance-heterogeneous (CDH) robots that inhibit each other, we define a new transformation rule between any two robot types. Building on the transformation, we design a novel utility function to measure the game's outcome quantitatively. Moreover, we rigorously prove the correctness of the designed utility function. Finally, we conduct extensive simulations to demonstrate the effectiveness of DOA on computing Nash equilibrium for homogeneous, linear heterogeneous, and CDH robot allocation on graphs.
\end{abstract}


\section{Introduction}
\label{sec:intro}
\IEEEPARstart{W}{ith} the advancement in computing, sensing, and communication, robots are increasingly used in various data collection tasks such as environmental exploration and coverage~\cite {zhou2023racer,tao2023seer,ahmadzadeh2008optimization,bhattacharya2014multi,ramachandran2020resilient,sharma2023d2coplan}, surveillance and reconnaissance~\cite{grocholsky2006cooperative,thakur2013planning}, and target tracking~\cite{dames2017detecting,zhou2018active,zhou2018resilient,zhou2019sensor,mayya2022adaptive,zhou2022graph,zhou2023robust}. These tasks are typically replete with rival competition. For example, in oil, ocean, and aerospace exploration, multiple competitors tend to deploy their specialized robots to cover and occupy sites of interest~\cite{reynolds2001oil}. This paper models this type of competition as a game-theoretic robot allocation problem, where two players allocate robots possessing defensive or offensive capabilities to compete for multiple interested sites. 

Game-theoretic robot allocation belongs to the class of resource allocation problems \cite{myerson1993incentives, kovenock2012coalitional,kovenock2010conflicts,berman2009optimized,prorok2017impact}. In the realm of game theory, these problems are typically modeled as the Colonel Blotto game, which was first introduced by Borel \cite{borel1921theorie} and further discussed in \cite{borel1953theory}. In the Colonel Blotto game, the two players allocate resources across multiple strategic points, with the player allocating more resources at a given strategic point being deemed the winner of that point. The seminal study from Borel and Ville has computed the equilibrium for a Colonel Blotto Game with three strategic points and equivalent total resources~\cite{borel1938applications}. In 1950, Gross and Wagner extended these findings to scenarios involving more than three strategic points~\cite{gross1950continuous}. Later, Roberson addressed the equilibrium problem for continuous Colonel Games under the premise of unequal total resources~\cite{roberson2006colonel}. Subsequent refinements to this continuous version emerged in \cite{behnezhad2017faster} and \cite{ahmadinejad2019duels}.
\begin{figure}[!tbp]
    \centering
    \includegraphics[width=2.9in]{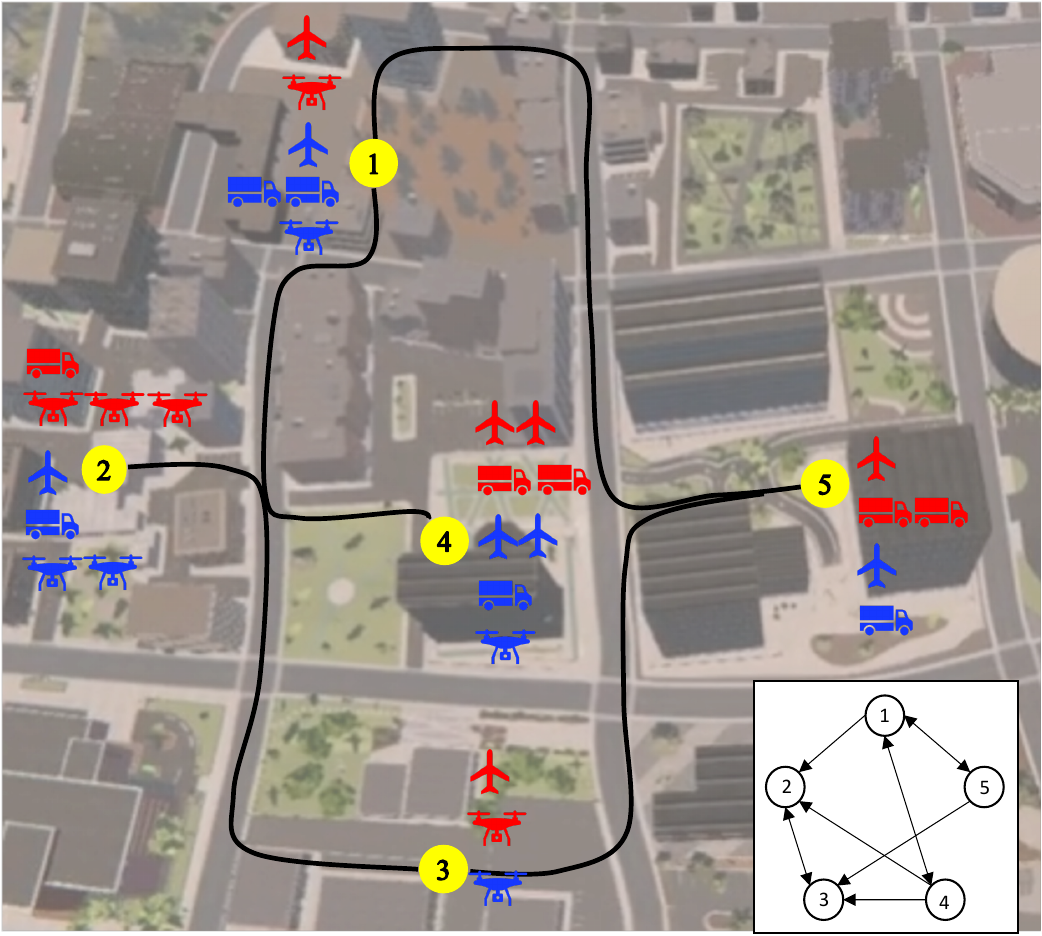}
    \caption{Game-theoretic robot allocation on a graph. Two players (``red" and ``blue") allocate three types of robots to compete for five sites in an area. The area is abstracted into a graph with five sites as the nodes and their connections as edges.} 
    \label{fig: map}
\end{figure}

Nevertheless, in addressing the problem of game-theoretic robot allocation, research on the conventional Blotto game exhibits several limitations. First, these studies have mainly focused on unilateral or homogeneous resources \cite{roberson2006colonel, behnezhad2017faster, ahmadinejad2019duels}. However, the robots can be heterogeneous and have different capabilities, which means the number of robots is not the only criterion of win and loss. Second, considering the environmental constraints and the motion abilities of the robots, there are sites between which no direct routes are accessible, e.g., the aerial robots are not allowed to enter no-fly zones and the wheeled robots may not be able to traverse terrains with sharp hills, tall grass, or mud puddles. Therefore, given these real-world constraints and limitations, the environment in which the robots operate is a non-fully connected graph where robots can transition only between connected sites. However, none of the previous research on resource allocation games involves different types of resources and graph constraints. Therefore, in this paper, we primarily focus on game-theoretic robot allocation on graphs with multiple types of robots (\figref{fig: map})\footnote{This problem could be more generic and go beyond robotics, such as the confrontation between different military forces in the war, the competition among different species populations in nature, and the investment over different financial products in economics.}.

First, the game requires both players to allocate robots under graph constraints, i.e., the robots can only be transitioned from one node to the other if there exists an edge between the two nodes. In practical scenarios, the resource allocation among nodes often demands consideration of the relative positions and connectivity of these points, as highlighted by \cite{shishika2022dynamic, jain2011double}, in discussions about offensive and defensive security games. 
However, in \cite{shishika2022dynamic}, the two players are asymmetric, with distinct roles as attacker and defender. The attacker's objective is to compromise critical nodes to gain access to a high-value target, with their win condition being the capture of any one of these nodes. Conversely, the defender's goal is to protect all critical nodes successfully. In contrast, our setting builds on the conventional Colonel Blotto game, where both players are symmetric. This means that they take actions simultaneously, and the player who controls more nodes emerges victorious.
In \cite{jain2011double}, the main focus is the network security game where the defender places resources to catch the attacker instead of a Colonel Blotto Game. Specifically, we consider distinct nodes to share unidirectional connections. Thus, robots cannot be directly transited between disconnected points. This setting morphs the allocation game into an asymmetric game, which is detailed further in the subsequent sections.

Second, we consider that both players allocate multiple types of robots that may inhibit each other, termed cyclic dominance heterogeneous
 (CDH) \cite{szolnoki2014cyclic}. In \cite{behnezhad2017faster}, the multiple-resource Colonel Blotto game has been introduced, but it considers different resources to be completely independent of each other. Different from that, For cyclic dominance heterogeneous (CDH) resources, any two types of resources hold a restrictive relationship as in the classic game of Rock-Paper-Scissors. This relationship is widely studied in evolutionary dynamics in biology science. For instance, Nowak et al.~ \cite{nowak2004evolutionary} utilized the examples of \emph{Uta stansburiana} lizard and \emph{Escherichia coli} to illustrate this mutually restrictive relationship. Similarly, an example of CDH robots could be three types of functional robots---cyber-security robots, network intrusion robots, and combat robots.
The cyber-security robot is primarily designed for defensive operations, specializing in counteracting cyber threats. It is equipped with advanced cybersecurity algorithms, enabling it to detect and neutralize attacks from network intrusion robots. 
The network intrusion robot excels in offensive cyber operations. This robot is capable of executing sophisticated network attacks, targeting vulnerabilities in other robots, particularly those relying on network-based controls such as combat robots.
The combat robot is engineered for physical combat and tactical defense. It is armed with state-of-the-art weaponry and robust armor, making it highly effective in direct physical confrontations such as with the cyber-security robot.
In other words, a cyber-security robot is capable of defending against network attacks from a network intrusion robot, but it can be physically destroyed by a combat robot. Meanwhile, a network intrusion robot can easily paralyze the network systems of a combat robot, causing it to crash, but it can be defeated by a security robot. In \cite{claussen2008cyclic}, this relationship is represented in matrix form and incorporated into the replicator dynamics. We use the concept of cyclic dominance to distinguish different types of mutual-restrictive robots. Based on it, we formulate a utility function to quantitatively evaluate the outcomes of the game.

We utilize the double oracle algorithm (DOA) to compute the equilibrium of the Colonel Blotto game on graphs. This algorithm was initially proposed by \cite{mcmahan2003planning} and has since been extensively adopted across various domains within game theory. In \cite{bosansky2012iterative} and \cite{bosansky2013double}, DOA was employed to solve two-player zero-sum sequential games with perfect information, with its performance verified on adversarial graph search and simplified variants of Poker. A theoretical guarantee on the convergence of the DOA was also provided. Additionally, Adam et al.~\cite{adam2021double} employed DOA for the conventional Colonel Blotto Game, juxtaposing it with the fictitious play algorithm and identifying a more rapid convergence rate for the former. Differently, we introduce graph constraints and CDH robot types into the Colonel Blotto game, which requires redefining the optimization equations and solving these newly defined equations in the DOA.

\textbf{Contributions.} We make three main contributions as follows: 
\begin{itemize}
    \item \textit{Problem formulation:} We formulate a game-theoretic robot allocation problem where two players allocate multiple types of robots on graphs. 
    \item \textit{Approach:} We leverage the DOA to calculate Nash equilibrium for the games with homogeneous, linear heterogeneous, and CDH robots. Particularly, for the CDH game, we introduce a new transformation rule between different types of robots. Based on it, we design a novel utility function and prove that the utility function is able to rigorously determine the winning conditions of two players. We also linearize the constructed utility function and formulate a mixed integer linear programming problem to solve the best response optimization problem in DOA. 
    \item \textit{Results:} We demonstrate that players' strategies computed by our approach converge to the Nash equilibrium and are better than other baseline approaches.
\end{itemize}

Overall, our main contribution is the construction of a novel utility function for the on-graph CBG with heterogeneous resources, and the application of DOA to calculate the equilibrium of this newly formulated game. To the best of our knowledge, we are the first to address the challenges of graph constraints and CDH robots in the field of the Colonel Blotto game. Furthermore, the simulation results validate that our approach achieves the Nash equilibrium.

\section{Conventions and Problem Formulation}
\label{sec:problem}
In this section, we first introduce game-theoretic robot allocation on graphs. The problem can be modeled as a Colonel Blotto game. Then we introduce the Colonel Blotto game including pure and mixed strategies and the equilibrium. Finally, we present the main problems to address in this paper. 

\subsection{Robot Allocation between Two Players on Graphs}
\label{ssec: graph}
The environment where two players allocate robots can be abstracted into a graph as shown in \figref{fig: map}. An ordered pair $G=(\mathcal{V},\mathcal{E})$ can be used to define an \emph{undirected graph}, comprising set of nodes $\mathcal{V}$ and set of edges $\mathcal{E}:=\{\{x,y\} \ | \ x,y\in \mathcal{V}\}$. Particularly, a \textit{complete graph} is a simple undirected graph in which every pair of distinct nodes is connected by a unique edge~\cite{bang2018basic}. Replacing nodes pair $\{x,y\}$ in $\mathcal{E}$ by ordered sequence of two elements $(x,y)$ in $\mathcal{V}$ leads to \emph{directed graph}, or \emph{digraph} with $\mathcal{E}:= \{(x,y)\ |\ (x,y) \in \mathcal{V}^2\}$ a set of edges. For an edge $\varepsilon=(x,y)$ directed from node $x$ to $y$, $x$ and $y$ are called the endpoints of the edge with $x$ the tail of $\varepsilon$ and $y$ the head of $\varepsilon$.  A directed graph $G$ is called \emph{unilaterally connected} or \emph{unilateral} if $G$ contains a directed path from $x$ to $y$ or a directed path from $y$ to $x$ for every pair of nodes $\{x,y\}$ \cite{gomez2003unilaterally}. \par
Previous studies on Colonel Blotto games do not incorporate graphs~\cite{behnezhad2017faster,ahmadinejad2019duels,adam2021double}. If the same logic is applied, they can be seen as the games conducted on a complete graph since there is no limitation on the robot transitions between any two nodes. In this paper, we focus on the Colonel Blotto games on unilateral graphs where the robots must be transitioned on the edges. 
Graph nodes are the strategic points where two players can allocate robots. We denote the number of nodes as $N$. 
We model the robot allocation by its two characteristics---\textit{type} denoted as $R$, and \textit{number} denoted as $A$. If there are $M$ robot types i.e., $R_1, R_2, \cdots, R_M$, the number of robots for each type can be denoted by $A_1, A_2, \cdots, A_M$, respectively. In the paper, $M=1$ denotes \textit{homogeneous robot allocation}  and $M>1$ denotes \textit{heterogeneous robot allocation}. To start the game of robot allocation, there must be an initial robot distribution for both players. Denote $\mathbf{d}_0^x(R_i)$ and $\mathbf{d}_0^y(R_i)$ as the initial distribution of $i$-th type robots from Player 1 and Player 2, respectively. Notably, $\mathbf{d}_0^x(R_i)$ (or $\mathbf{d}_0^y(R_i)$) is a vector of $N$ entries, each of which stands for the $i$-th type robots allocated on a particular node. \figref{fig: homo_ex} demonstrates an example of homogeneous robot allocation. The initial robot distribution of Player 1 (red)  is $[1,0,2,2,0]$. Then she translates one robot from node 3 to node 2 and another from node 4 to node 1.  The robot distribution becomes $[2,0,2,1,0]$. Since $(1,4),(1,5)\in \mathcal{E}$, this allocation is valid. Meanwhile, the robot allocation of Player 2 (blue) after one step is $[1,1,0,2,1]$. If the rule stipulates that the player who owns more robots on a node wins that node, then the outcome is that Player 2 wins. That is because she wins nodes $2,4,5$ while Player 1 only wins nodes $1,3$. \par
\begin{figure}[!tbp]
    \centering
    \subfloat[Homogeneous robot allocation at time steps $t$ and $t+1$]{%
    \label{fig: homo_ex}
    \includegraphics[width=2.9in]{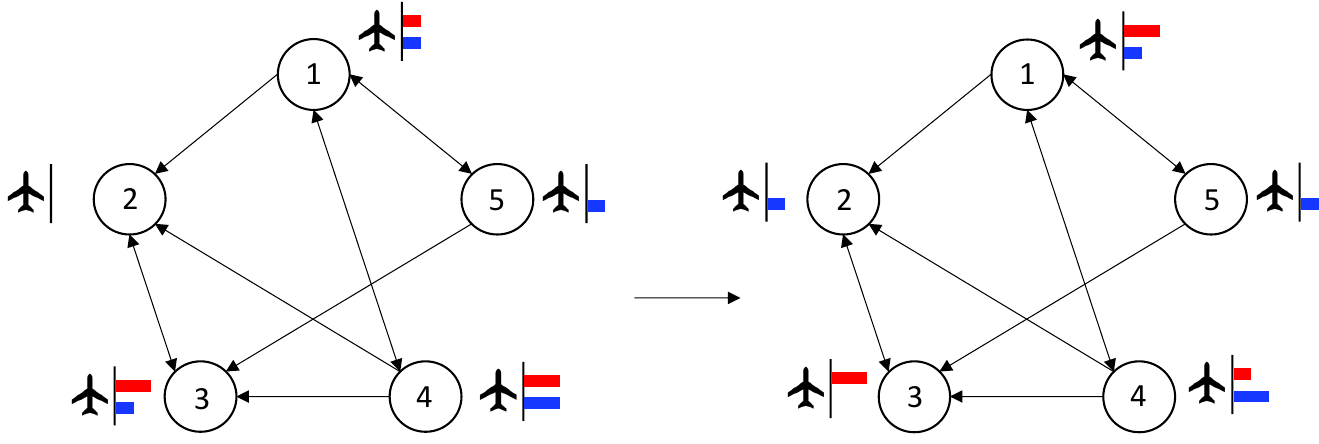}
    }%
    
    \subfloat[Heterogeneous robot allocation for two players with Player 1's mixed strategy and Player 2's pure strategy]{%
    \label{fig: heter_ex}
    \includegraphics[width=2.9in]{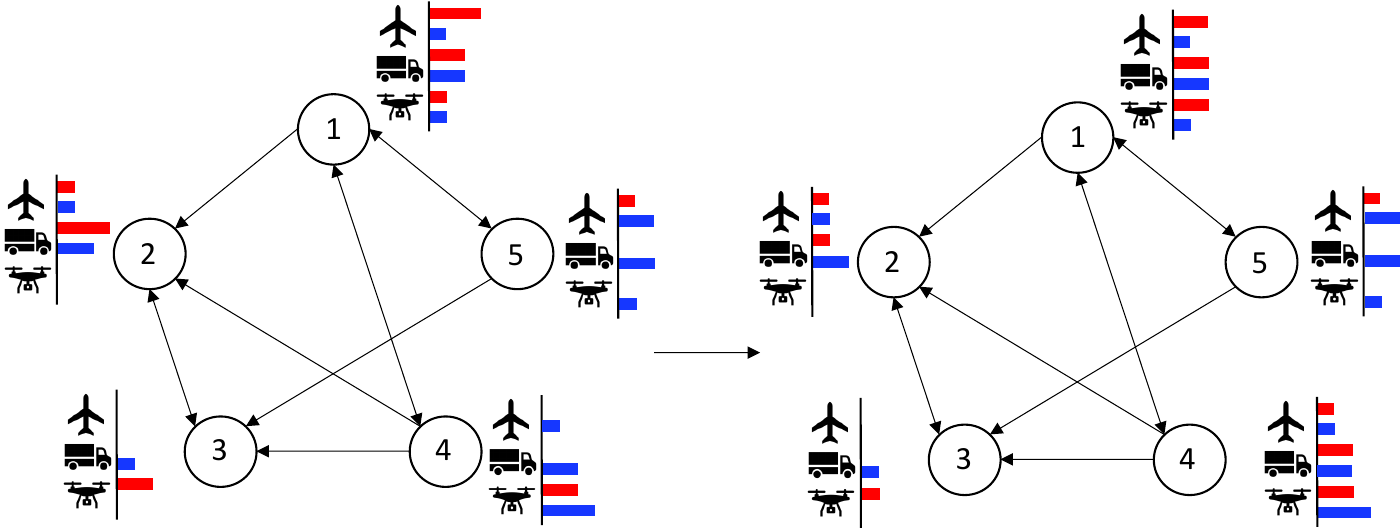}
    }%
    \caption{Examples of allocation games with (a) homogeneous and (b) heterogeneous robots. The number of robots allocated on each node is demonstrated as the length of the color bar, with red representing Player 1, and blue representing Player 2. In (a), at time $t$ (left), Player 1 wins node 3 and Player 2 wins node 5. At time $t+1$ (right). Player 1 moves one robot from node 4 into node 1, and Player 2 moves one robot from node 3 to node 2. Thus Player 1 wins nodes 1 and 3 while Player 2 wins the remaining nodes. In (b), Player 1 adopts mixed strategy $\mathbf{\Delta}_X \sim \text{Binomial}(0.4, 0.6)$ and Player 2 adopts pure strategy $\mathbf{\Delta}_Y=\mathbf{S}_y^1$. The win or loss cannot be claimed solely based on the number of robots because of the CDH robots.}
    \label{fig: ex} 
\end{figure}
\subsection{Colonel Blotto Game}
\label{ssec: equilibrium}
A \emph{Colonel Blotto Game} is a two-player constant-sum game in which the players are tasked to simultaneously allocate limited resources over several strategic nodes \cite{borel1921theorie}. Each player selects strategies from a nonempty strategy set $\mathcal{X}=\mathcal{X}(G),~\mathcal{Y}=\mathcal{Y}(G)$, where $\mathcal{X},~\mathcal{Y} \subseteq \{\mathbf{S}_{M\times N}\ |\ 0\leq \mathbf{S}_{ij} \leq A_i, \sum_j \mathbf{S}_{ij}=A_i , i\in \{1,2,\cdots,M\}, j\in \{1,2,...,N\}\}$. $\mathbf{S}_{ij}$ stands for amount of $i$-th resource allocated on $j$-th node. Each player picks their strategies combination $\hat{\mathcal{X}}\subseteq \mathcal{X}, \hat{\mathcal{Y}}\subseteq \mathcal{Y}$ over probability distribution $\mathbf{P}_x$ and $\mathbf{P}_y$, respectively. The probability of strategy $\mathbf{S}_x$ adopted by Player 1 is denoted as $P_x(\mathbf{S}_x)$. Analogously for Player 2, it is denoted as $P_y(\mathbf{S}_y)$.  This strategy's combination over the probability distribution is called \textit{mixed strategy}, denoted by $\mathbf{\Delta}_X$ and $\mathbf{\Delta}_Y$. The number of strategies in a mixed strategy set is denoted by $K$. For example, if Player 1 decides to apply the mixed strategy of $K_x$ strategies from $\mathcal{X}$, then $\mathbf{\Delta}_X = \mathbf{P}_x\times \mathbf{S}_x \in \mathbb{R}^{M\times (N\times K_x)}$. If $K=1$, the strategy is called \emph{pure strategy}. \figref{fig: heter_ex} illustrates an example of heterogeneous robot allocation that expands from the case of homogeneous robot allocation (\figref{fig: homo_ex}). There are three types of robots, i.e., $M=3$ and each type has five robots, i.e., $A_1=A_2=A_3=5$. There are five nodes in the graph, i.e., $N=5$. For Player 1, 
\begin{equation*}
    \mathbf{S}_x^1 =\begin{bmatrix}
        3 & 1 & 0 & 0 & 1\\
        2 & 3 & 0 & 0 & 0\\
        1 & 0 & 2 & 2 & 0\\
    \end{bmatrix},\quad
    \mathbf{S}_x^2 =\begin{bmatrix}
        2 & 1 & 0 & 1 & 1\\
        2 & 1 & 0 & 2 & 0\\
        2 & 0 & 1 & 2 & 0\\
    \end{bmatrix}.
\end{equation*}
and for Player 2,
\begin{equation*}
    \mathbf{S}_y^1 =\begin{bmatrix}
        1 & 1 & 0 & 1 & 2\\
        2 & 2 & 0 & 1 & 0\\
        1 & 0 & 0 & 3 & 1\\
    \end{bmatrix}.
\end{equation*}
If the probability of taking the first strategy and the second strategy is respectively 0.4 and 0.6, i.e., $P_x(\mathbf{S}_x^1)=0.4, P_x(\mathbf{S}_x^2)=0.6$,
then mixed strategy for Player 1 is a Binomial Distribution $$(\mathbf{S}_x^1, \mathbf{S}_x^2)\sim \text{Binomial}(0.4, 0.6).$$
For the sake of brevity, we denote this distribution as $\mathbf{\Delta}_X$.
While Player 2 takes pure strategy with $\mathbf{P}_y(\mathbf{S}_y^1)=1$, thus $\mathbf{\Delta}_Y=\mathbf{S}_y^1$.
Notably, the mixed strategy comprises elements from compact strategy sets $\mathcal{X}$ and $\mathcal{Y}$. For each strategy element $\mathbf{S}_x$ and $\mathbf{S}_y$, the \emph{utility function} is a continuous function measuring the outcome of the game $u=u(\mathbf{S}_x, \mathbf{S}_y): \mathcal{X}\times \mathcal{Y} \rightarrow \mathbb{R}$ for Player 1, while for Player 2 utility function is $-u$ if the game is zero-sum. In classic Colonel 
Blotto game (i.e., homogeneous robot allocation), the utility function can be modeled as a sign function since the player allocating more robots to a node wins that node. The utility is the total number of strategic points the player wins, as shown in Equation~\ref{eq:utility}. 
\begin{align}
    \label{eq:utility}
     u(\mathbf{S}_x, \mathbf{S}_y) = \sum_{i=1}^{N}{\texttt{sgn}(\mathbf{S}_{x,i} - \mathbf{S}_{y,i})}.
\end{align}

\begin{align*}
\text{where}\quad \texttt{sgn}(x) = \left\{
        \begin{aligned}
        -1&, &x&\leq-C;\\
        x/C&, &x&\in [-C, C ];\\
        1&, &x&\geq C.
        \end{aligned}
        \right.
\end{align*}
$C$ is a positive constant to maintain the continuity of the function. Here $\mathbf{S}_x$ and $\mathbf{S}_y$ are both $N$-dimension vectors since the classic Colonel Blotto game is a homogeneous game where $M=1$ (mentioned in section \ref{ssec: graph}), and $\mathbf{S}_{x, i}, \mathbf{S}_{y, i}$ refer to $i$-th entry of $\mathbf{S}_x$ and $\mathbf{S}_y$. However, in games with heterogeneous resources, the utility functions can be more complicated.\par

The triple $\mathbb{C}=(\mathcal{X}(G),\mathcal{Y}(G),u)$ is called a \emph{continuous game}. In a continuous game, given two players' mixed strategies
$\mathbf{\Delta}_X$ and $\mathbf{\Delta}_Y$, \emph{expected utility} of Player 1 (expected loss of Player 2) is defined as 

\begin{align*}
U(\mathbf{\Delta}_X, \mathbf{\Delta}_Y) &= \sum_{\mathbf{S}_x\in \hat{\mathcal{X}}}\sum_{\mathbf{S}_y\in \hat{\mathcal{Y}}}P_x(\mathbf{S}_x)P_y(\mathbf{S}_y)u(\mathbf{S}_x, \mathbf{S}_y).
\end{align*}

\begin{lemma}
\itshape
\label{lemma:equilibrium}
A mixed strategy group $(\mathbf{\Delta}_X^*, \mathbf{\Delta}_Y^*)$ is \emph{equilibrium} of a continuous game $\mathbb{C}$ if, for all $(\mathbf{\Delta}_X, \mathbf{\Delta}_Y)$,
    \begin{align*}
        U(\mathbf{\Delta}_X, \mathbf{\Delta}_Y^*) \leq U(\mathbf{\Delta}_X^*, \mathbf{\Delta}^*_Y) \leq  U(\mathbf{\Delta}_X^*, \mathbf{\Delta}_Y).
    \end{align*}
\end{lemma}

According to \cite{glicksberg1952further}, every continuous game has at least one equilibrium.
\subsection{Problem Formulation}
\label{ssec: problem formulation}
We consider a two-player game-theoretic robot allocation on a connected directed graph $G=(\mathcal{V}, \mathcal{E})$ with $| \mathcal{V} |=N$. Suppose each player has $M$ types of robots $R_1, R_2, \cdots, R_M$ with corresponding quantity $A_1, A_2, \cdots, A_N$. 
We consider the game from timestep $t$ to timestep $t+1$. Without loss of generality, we set the robot distribution at time $t$ as a pure distribution, denoted as $\mathbf{d}_x$ and $\mathbf{d}_y$. The initial distribution refers to the distribution at the current time $t$ throughout the rest of the paper.
We first introduce three basic assumptions.
\begin{assumption}
    We assume that the transition of robots between two connected nodes can be accomplished within one step. 
\end{assumption}
\begin{assumption}
    We assume that both players know each other's current robot distribution.
\end{assumption}
\begin{assumption}
    We assume that on all nodes, robots of the same type engage in combat first, followed by combat between different types of robots.
\end{assumption}
We introduce the third assumption based on the rules of the conventional Colonel Blotto game that involves only a single type of robot. A more detailed explanation will be provided in Section~\ref{ssec:NLH-heter}. With these three assumptions, we aim to compute the equilibrium (i.e., optimal mixed strategies for the two players) for two versions of the game-theoretic robot allocation on graphs.

\begin{problem}[\textit{Allocation game with homogeneous and linear hetergeneous robots}]\label{prob:1}
     All robots are homogeneous (i.e., $M=1$). The win or loss at a given node is determined purely by the number of robots allocated. The utility function for the players is introduced in \eqref{eq:utility}, where $C$ is a continuous factor ensuring the smoothness and continuity of the function. This can be extended into the case of robot allocation with linear heterogeneity, i.e., $M>1$ and there exists a linear transformation between different robot types, e.g., $R_i = a R_j$ with $a$ a positive constant.
\end{problem}


\begin{problem}[\textit{Allocation game with CDH robots}]\label{prob:3}
    Robots are heterogeneous, i.e., $M>1$, and they mutually inhibit one another. This is commonly referred to as cyclic dominance~\cite{szolnoki2014cyclic}.
    In CDH robot allocation, the number of robots on a node can no longer determine the winning condition of that node since the combination of heterogeneous robots also plays an important role. Therefore new utility function is required to describe this inhibiting relationship, which is introduced in \secref{ssec:NLH-heter}.
\end{problem} 

\section{Preliminaries}
\label{sec:pre}
\subsection{Double Oracle Algorithm}
\label{ssec:do-alg}

\begin{algorithm}[!t]
    \caption{Double Oracle Algorithm}
    \textbf{Input: }
    \begin{itemize}
        \item Graph $G=(\mathcal{V},\mathcal{E})$; 
        \item Continuous game $\mathbb{C}=(\mathcal{X}(G),\mathcal{Y}(G),u)$;
        \item Initial nonempty strategy set $\mathcal{X}_0 \subseteq \mathcal{X}, \mathcal{Y}_0 \subseteq \mathcal{Y}$;
        \item threshold $\epsilon$.
    \end{itemize}

    \begin{algorithmic}[1]
        \State $i\leftarrow 0$;
        \While{$U_u-U_l>\epsilon$}
            \State $(\mathbf{\Delta}_X^{i*}, \mathbf{\Delta}_Y^{i*}) \leftarrow (\mathcal{X}_i, \mathcal{Y}_i, u)$ calculate equilibrium for subgame;\label{alg1.1}
            \State $(\mathbf{\Delta}_X^{i*}, \mathbf{\Delta}_Y^{i*})$ reject strategies with low probability;\label{alg1.3}        
            \State $\mathbf{S}_x^{i+1}, \mathbf{S}_y^{i+1} \leftarrow (\mathbf{\Delta}_X^{i*}, \mathbf{\Delta}_Y^{i*})$ find best response strategy;\label{alg1.2}
            \If{$\mathbf{S}_x^{i+1}\notin \mathcal{X}_i$ and $\mathbf{S}_y^{i+1}\notin \mathcal{Y}_i$}
                \State $\mathcal{X}_{i+1} = \mathcal{X}_i\ \bigcup \ \{\mathbf{S}_x^{i+1}\}, \mathcal{Y}_{i+1} = \mathcal{Y}_i\ \bigcup \ \{\mathbf{S}_y^{i+1}\}$;
                \State $U_l:=U(\mathbf{\Delta}_X^{i*}, \mathbf{S}_y^{i+1}), U_u:=U(\mathbf{S}_x^{i+1}, \mathbf{\Delta}_Y^{i*})$;
            \EndIf
            \State $i\leftarrow i+1$;
        \EndWhile
    \end{algorithmic}
    \textbf{Output: }
    \begin{itemize}
        \item $\epsilon$-equilibrium mixed strategy set $(\mathbf{\Delta}_X^*, \mathbf{\Delta}_Y^*)$ of game $\mathbb{C}$; 
    \end{itemize}
    \label{alg:doa}
\end{algorithm}
The double oracle algorithm was first presented in \cite{mcmahan2003planning}. Later, it was widely utilized to compute the equilibrium (or mixed strategies) for continuous games. \cite{jain2011double, bosansky2014exact}. To better understand the algorithm, we first introduce the notion of \emph{best response strategy}. The best response strategy of Player 1, given the mixed strategy $\mathbf{\Delta}_Y$ of Player 2, is defined as 
\begin{align}
\label{eq:bst_res_x}
    \delta_x(\mathbf{\Delta}_Y):=\{
    \mathbf{S}_x \in \mathcal{X} \ |\ U(\mathbf{S}_x, \mathbf{\Delta}_Y)=\max_{\mathbf{S}'_x \in \mathcal{X}} U(\mathbf{S}'_x, \mathbf{\Delta}_Y)
    \}.
\end{align}
It is Player 1's best pure strategy against her opponent (i.e., Player 2). Similarly, Player 2's best response strategy is
\begin{align}
\label{eq:bst_res_y}
    \delta_y(\mathbf{\Delta}_X):=\{
    \mathbf{S}_y \in \mathcal{Y} \ |\ U(\mathbf{\Delta}_X, \mathbf{S}_y)=\min_{\mathbf{S}'_y \in \mathcal{Y}} U(\mathbf{\Delta}_X, \mathbf{S}'_y)
    \}.
\end{align}

The double oracle algorithm begins from an initial strategy set $\mathcal{X}_0$ and $\mathcal{Y}_0$ for the two players, which are picked randomly from strategy space $\mathcal{X}$ and $\mathcal{Y}$. $\mathbb{C}_0=(\mathcal{X}_0, \mathcal{Y}_0, u)$ is a subgame of $\mathbb{C}=(\mathcal{X}, \mathcal{Y},u)$, and the equilibrium of this subgame can be calculated by linear programming, as in \algref{alg:doa}, \linref{alg1.1}. From the equilibrium of the subgame, notated as $\mathbf{\Delta}_X^{i*}$ and $\mathbf{\Delta}_Y^{i*}$ in $i$-th iteration, best response strategies for both players are computed as in \algref{alg:doa}, \linref{alg1.2}. If these strategies are not in the subgame strategy set, then add them into the set to form a bigger subgame, and calculate whether equilibrium is reached. According to \cite{adam2021double}, the equilibrium condition (Lemma~\ref{lemma:equilibrium}) is equivalent to the following lemma.
\begin{lemma}
\itshape
    A mixed strategy group $(\mathbf{\Delta}_X^*, \mathbf{\Delta}_Y^*)$ is  \emph{equilibrium} of a continuous game $\mathbb{C}$ if, for all $\mathbf{S}_x\in \delta_x(\mathbf{\Delta}^*_Y)$ and $\mathbf{S}_y\in \delta_y(\mathbf{\Delta}^*_X)$,
    \begin{align}\label{eq:equilibrium}
        U(\mathbf{S}_x, \mathbf{\Delta}_Y^*) = U(\mathbf{\Delta}_X^*, \mathbf{\Delta}_Y^*) = U(\mathbf{\Delta}_X^*, \mathbf{S}_y).
    \end{align}
    \label{lemma_2}
\end{lemma}
$\max{U(\mathbf{\Delta}_X^*, \mathbf{S}_y)}$ is called upper utility of the game, denoted by $U_u$, and $\min{U(\mathbf{S}_x, \mathbf{\Delta}_Y^*)}$ is called lower utility of the game, denoted by $U_l$. Note that, in practice, \eqref{eq:equilibrium} is typically replaced by $\epsilon$-equilibrium equation for the simplicity of computation.
\begin{align*}
\label{eq:e-equilibrium}
 U(\mathbf{S}_x, \mathbf{\Delta}_Y^*)-\epsilon \leq U(\mathbf{\Delta}_X^*, \mathbf{\Delta}_Y^*) \leq  U(\mathbf{\Delta}_X^*, \mathbf{S}_y)+\epsilon.
\end{align*}
Therefore, the termination condition is $U_u-U_l>\epsilon$. In addition, since the output is the mixed strategy and too many redundant strategies with a low probability would slow down the algorithm, we delete strategies with a low probability, as shown in \algref{alg:doa}, \linref{alg1.3}.

The optimization problem in \algref{alg:doa}, \linref{alg1.1} is a linear programming (LP) problem. Define the \emph{utility matrix} over $\mathbf{S}_x$ and $\mathbf{S}_y$ by $\mathbf{U}$, with each entry $\mathbf{U}_{ij}$ being the utility of Player 1's $i$-th strategy verses Player 2's $j$-th strategy, i.e.,
\begin{align*}
    \mathbf{U}_{ij}=u(\mathbf{S}_x^i,\mathbf{S}_y^j). 
\end{align*}
Therefore, the expected utility can be calculated by
\begin{align*}
    U(\mathbf{\Delta}_X, \mathbf{\Delta}_Y)=\mathbf{P}_x^T\mathbf{U}\mathbf{P}_y.
\end{align*}
Suppose $\mathbf{U}$ is a $K_x\times K_y$ matrix (that is, Player 1's mixed strategy set containing $K_x$ strategies and Player 2's mixed strategy set containing $K_y$ strategies), the probability distribution of Player 1's strategies $\mathbf{P}_x$ can be calculated by
\begin{align}
    \begin{aligned}
            &\min_{\mathbf{P}_x,\ U_x\in\mathbb{R}}\ U_x\\
        s.t.\ 
            &\sum_{i=1}^{{K_x}}P_x(\mathbf{S}_x^i)=1,\\
            &P_x(\mathbf{S}_x^i)\geq0,\ \text{for}\ i\in\{1,2,\cdots,{K_x}\},\\
            &U_x\geq\sum_{i=1}^{K_x}P_x(\mathbf{S}_x^i)\mathbf{U}_{ij},\ \text{for}\ j\in\{1,2,\cdots,{K_y}\}.
    \end{aligned}
\end{align}
For Player 2, the probability distribution $\mathbf{P}_y$ can be calculated by
\begin{align}
    \begin{aligned}
            &\max_{\mathbf{P}_y,\ U_y\in\mathbb{R}}\ U_y\\
        s.t.\ 
        &\sum_{j=1}^{{K_y}}P_y(\mathbf{S}_y^j)=1,\\
        &P_y(\mathbf{S}_y^j)\geq0,\ \text{for}\ j\in\{1,2,\cdots,{K_y}\},\\
&U_y\leq\sum_{j=1}^{K_y}P_y(\mathbf{S}_y^j)\mathbf{U}_{ij},\ \text{for}\ i\in\{1,2,\cdots,{K_x}\}.
    \end{aligned}
\end{align}
These can be solved by \texttt{linprog} in Python.

The most crucial part of the double oracle algorithm is to compute the best response strategy (\algref{alg:doa}, \linref{alg1.2}). The graph constraints and various robot types we consider make the computation challenging. First, the strategy set $\mathcal{X}$ and $\mathcal{Y}$ in \eqref{eq:bst_res_x} and \eqref{eq:bst_res_y} depends on the graph $G$. Given that $G$ is unilaterally
connected, not all nodes can be reached within one step. Therefore, with a robot distribution $\mathbf{d}_x, \mathbf{d}_y$, reachable strategy sets $\mathcal{X}$ and $\mathcal{Y}$ need to be computed first. This problem was introduced in paper~\cite{shishika2022dynamic}, which we briefly summarize in Section~\ref{ssec:reachable set}. Second, an appropriate utility is essential for computing the best response strategy. However, for CDH robots, no such function exists. To this end, we construct a continuous utility function to handle CDH robots (\secref{ssec:NLH-heter}).
\subsection{Reachable Strategy Set}
\label{ssec:reachable set}
We leverage the notion of \emph{reachable set} for computing strategy sets $\mathcal{X}(G)$ and $\mathcal{Y}(G)$ on graph $G$. We define the \emph{adjacency matrix} $\mathbf{A}$  of a graph $G$ as:
\begin{equation*}
\label{eq:adjacency}
        \mathbf{A}_{ij} = \left\{
        \begin{aligned}
        0&, &(j,i)&\notin \mathcal{E};\\
        1&, &(j,i)&\in \mathcal{E}.
        \end{aligned}
        \right.
\end{equation*}
The adjacency matrix describes the connectivity of nodes on a graph. We denote the $i$-th type robots $R_i$ distributed on a graph at time $t$ by a $N$-dimension vector $\mathbf{d}_t(R_i)$ with $N$ the number of nodes and each element the number of this type robots on each node. Here we use the fraction of the robot population instead of the number of robots to represent the allocation amount. Therefore, for a distribution $\mathbf{d}$ with its $j$-th element denoted as $\mathbf{d}^j$, it satisfies $\sum_{j=1}^N\mathbf{d}^j=1,\ \mathbf{d}^j\geq 0.$ Then, we denote the \emph{Transition matrix} $\mathbf{T}$ as
\begin{equation*}
\label{eq:transition}
    \mathbf{d}_{t+1}(R_i) = \mathbf{T}\mathbf{d}_t(R_i).
\end{equation*}
Note that the transition matrix $\mathbf{T}$ has two characteristics.
\begin{itemize}
    \item $\sum_j\mathbf{T}_{ij}=1.$
    \item $\mathbf{T}_{ij}\geq 0, \ \mathbf{T}_{ij}=0 \quad \text{if} \  \mathbf{A}_{ij}=0.$
\end{itemize}
All valid transition matrices form \emph{admissible action space} $\widetilde{\mathcal{T}}$, i.e.,
\begin{equation*}
\label{eq:admissible}
    \widetilde{\mathcal{T}}=\{\mathbf{T}\in \mathbb{R}^{N\times N} \ |\ \sum_j\mathbf{T}_{ij}=1,\ \mathbf{T}_{ij}\geq 0, \ \mathbf{T}_{ij}=0\ \text{if} \  \mathbf{A}_{ij}=0\}.
\end{equation*}
To determine closed-form expression of $\widetilde{\mathcal{T}}$, define \emph{extreme action space} $\hat{\mathcal{T}}$ corresponding to the binomial allocation as
\begin{equation}
\label{eq:extreme}
    \hat{\mathcal{T}} = \{\mathbf{T}\in \mathbb{R}^{N\times N} \ |\ \mathbf{T}\in \widetilde{\mathcal{T}}, \mathbf{T}_{ij}\in \{0,1\} \}.
\end{equation}
Extreme action space $\hat{\mathcal{T}}$ is a finite set and its element forms the base of $\widetilde{\mathcal{T}}$. Therefore, admissible action space can be calculated as a linear combination of all elements in extreme action space. 
\begin{equation}
\label{eq:cls-ad}
    \widetilde{\mathcal{T}} = \{\mathbf{T}\in \mathbb{R}^{N\times N} \ |\ \mathbf{T}=\sum_{\mathbf{T}'\in \hat{\mathcal{T}}}\lambda_i\mathbf{T}', \sum_i\lambda_i=1,\lambda_i\in\mathbb{R}^+ \}.
\end{equation}
Then, given a distribution of the $i$-th robot type, $\mathbf{d}(R_i)$, the \emph{reachable set} for $i$-th robot type can be calculated as
\begin{equation}
\label{eq:reachable}
    \mathcal{R}(R_i) = \{\mathbf{d}\ |\ \mathbf{T}\in \widetilde{\mathcal{T}}, \mathbf{d}=\mathbf{T}\mathbf{d}(R_i) \}.
\end{equation}
Finally, the strategy set $\mathcal{X}$ for all types of robots can be expressed as a collection of their individual reachable sets. 
\begin{equation}
\label{eq:strategy set}
    \mathcal{X} = \{\mathbf{S}\in \mathbb{R}^{M\times N}\ |\ \mathbf{S}= [\mathbf{R}_1,\ \cdots,\mathbf{R}_M]^T, \mathbf{R}_i\in \mathcal{R}(R_i)\}.
\end{equation}
The reachable set $\mathcal{Y}$ can be calculated analogously.
 
\section{Approach}
\label{sec:approach}

In this section, we present solutions to compute the equilibrium (i.e., optimal mixed strategies for the two players) for Problem~\ref{prob:1} and Problem~\ref{prob:3}. 

\subsection{Approach for Problem~\ref{prob:1}}
\label{ssec:homo}
In Problem~\ref{prob:1}, two players allocate homogeneous robots (i.e., $M=1$) on graph $G$. We denote initial robot distribution as $\mathbf{d}_x$ and $\mathbf{d}_y$. We first calculate the reachable set of initial strategy sets $\mathcal{X}(\mathbf{d}_x)$ and $\mathcal{Y}(\mathbf{d}_y)$ by \eqref{eq:reachable} and \eqref{eq:strategy set}. 
\begin{align}
\begin{split}
\label{eq:reachable_set_for_homo}
    \mathcal{X}(\mathbf{d}_x) &= \{\mathbf{S}\in \mathbb{R}^{N}\ |\ \mathbf{T}\in \widetilde{\mathcal{T}}, \mathbf{S}=\mathbf{T}\mathbf{d}_x\},\\
    \mathcal{Y}(\mathbf{d}_y) &= \{\mathbf{S}\in \mathbb{R}^{N}\ |\ \mathbf{T}\in \widetilde{\mathcal{T}}, \mathbf{S}=\mathbf{T}\mathbf{d}_y\}.
\end{split}
\end{align} 
Then we utilize the Double Oracle algorithm (\algref{alg:doa}) to compute the optimal mixed strategies for the two players with the utility function introduced in \eqref{eq:utility}. Particularly, assume in $j$-th step $\mathbf{\Delta}_y^{j*}=[P_y(\mathbf{S}_y^1)\cdot \mathbf{S}_y^1, \ P_y(\mathbf{S}_y^2)\cdot \mathbf{S}_y^2,\ \cdots, \ P_y(\mathbf{S}_y^K)\cdot \mathbf{S}_y^K]$ is known, the best response strategy for Player 1 is determined by  
\begin{equation}
\label{eq:bst_st}
    \max_{\mathbf{S}_x \in \mathcal{X}(\mathbf{d}_x)}\sum_{i=1}^{K_y}P_y(\mathbf{S}_y^i)u(\mathbf{S}_x,\mathbf{S}_y^i).
\end{equation}
The approach to solve this problem will be discussed in \secref{ssec:doa_cdh}.


In the case of linear heterogeneity, two players allocate linear heterogeneous robots (i.e., $M>1$) on graph $G$. Then the strategy space $\mathbf{S}\in\mathbb{R}^{M\times N}$ becomes a $M\times N$-dimension matrix. 
Different robot types can be linearly transformed, i.e., $R_i=I_{ij}R_j$ by an intrinsic matrix $\mathbf{I}$, i.e.,
\begin{subequations}
\label{intrinsic cons}
\begin{align}
    &\mathbf{I}_{ij}\mathbf{I}_{ji}=1, \;\text{for} \; i\neq j, \;\mathbf{I}_{ii}=1,\label{9.a}\\
    &\mathbf{I}_{ik}\mathbf{I}_{kj}=\mathbf{I}_{ij}, \;\text{for any} \; i,j,k\in\{1,2,\cdots,M\}.\label{9.b}
\end{align}
\end{subequations}
The second equality ensures that the transformation between any two types of robots is reversible and multiplicative. Therefore, all types can be transformed to a specified type $f$, i.e., $R_f = \mathbf{I}_{fi}R_i$ for all $i\neq f, i\in\{1,2,\cdots, M\}$. 
This implies that the linear heterogeneous robots can be transformed into homogeneous robots. In other words, the linear heterogeneity can be converted to homogeneity. Specifically, after transformation, strategy space $\mathbf{S}'\in \mathbb{R}^N$ and initial robot distributions $\mathbf{d}'_x, \mathbf{d}'_y$ become $N$-dimension vectors. Then the reachable set of initial strategy sets $\mathcal{X}(\mathbf{d}'_x)$ and $\mathcal{Y}(\mathbf{d}'_y)$ can be calculated by \eqref{eq:reachable_set_for_homo}.
Therefore, if the relationships between robots are linear heterogeneous, the approach for homogeneous robot allocation can be directly applied.

\subsection{Approach for Problem~\ref{prob:3}}
\label{ssec:NLH-heter}
In Problem~\ref{prob:3}, two players allocate cyclic-dominance-heterogeneous (CDH) robots on graph $G$. In this paper, we specify that if robot type $R_i$ dominates robot type $R_j$, then $R_i$ is capable of neutralizing or overcoming a greater number of $R_j$ than its own quantity. With CDH robots, the linear intrinsic transformation between robot types is broken and not guaranteed to be closed anymore. For example, if $M=3$ and $R_1=2R_2, R_2=2R_3, R_3=2R_1$, by circularly converting $R_1$ to $R_2$, $R_2$ to $R_3$, and $R_3$ to $R_1$, the number of $R_1$ increases out of thin air. Therefore, we design a new transformation rule, named \textit{elimination transformation} between different types of CDH robots.

\noindent \textbf{Elimination transformation.} In CDH robot allocation, transformation is valid only when robots can be eliminated but not created between two players. For example, consider $\mathbf{I}_{12}=\mathbf{I}_{23}=\mathbf{I}_{31}=2$, and Player 1 allocates $\mathbf{S}_{x,i}=(R_1, 2R_2, 4R_3)$ and Player 2 allocates $\mathbf{S}_{y,i}=(3R_1, R_2,3R_3)$ on node $i$. According to \textit{Assumption 3}, we first cancel out the same types of robots, and the remaining robot distribution after subtraction on node $i$ is $\mathbf{S}_{x,i}-\mathbf{S}_{y,i}=(-2R_1, R_2, R_3)$. This means Player 1 wins the game since Player 1 can consume Player 2's $2R_1$ at the cost of $R_3$, and still has one $R_2$ to win node $i$. However, if we transform Player 2's $R_1$ to $R_2$, i.e., $-2R_1=-4R_2$, the remaining robots are $-4R_2 + R_2 + 0.5R_2 = -2.5 R_2$ with $0.5 R_2$ transformed from $R_3$, leading to Player 2's win. This process is wrong because we cannot transform $R_1$ to $R_2$ as we cannot create new type-2 robots out of type-1 robots. Therefore, robots can only be eliminated but not created. Following this rule, we can transform $-0.5R_1$ to cancel out $R_2$ and the remaining $-1.5R_1$ can be eliminated by $0.75 R_3$. In the end, we have the remaining robots as $0.25 R_3$, and thus Player 1 wins.

Based on \textit{Assumption 3}, we require robots of the same type to engage in combat first—meaning they are eliminated (the same as the setup in the conventional CBG~\cite{roberson2006colonel}) before considering interactions between different types. Without this same-type elimination, each player will intentionally use their robot types to attack the other's robot types that they dominate, leading to ambiguity in the game outcome. By eliminating identical robot types first, we can ensure clarity and consistency in the game outcome. For instance, suppose Player 1 allocates (2$R_1$, 2$R_2$, 2$R_3$), while Player 2 allocates (3$R_1$, 3$R_2$, 3$R_3$) on a node. 
If two players execute the simultaneous move (a property of Colonel Blotto game), Player 2 is the winner since she has more robots for each type.
However, without same-type elimination, Player 1 could manipulate this situation: he could use 1.5 $R_1$ to consume 3 $R_2$ of Player 2, 1.5 $R_2$ to consume 3 $R_3$ of Player 2, and 1.5 $R_3$ to consume 3 $R_1$ of Player 2. Then Player 1 would have (0.5$R_1$, 0.5$R_2$, 0.5$R_3$) left, resulting in an illogical win. In this case, Player 1 wins, which is clearly unreasonable, demonstrating that without same-type elimination, the game's result becomes ambiguous.
Therefore, the same-type elimination rule preserves the simultaneous move property and ensures a unique and logically consistent outcome\footnote{Same-type elimination is not the only way that ensures a unique outcome. For example, enforcing a fixed battle order between different types of robots can also guarantee uniqueness. However, this is out of the scope of this paper. We will investigate it in our future work.}. 

Based on the elimination transformation rule, we construct a novel utility function to decide on the outcome of the game. In particular, we consider three types of CDH robots (i.e., $M=3$) as in the case of Rock-Paper-Scissor.~\footnote{For $M\geq 4$, the operation becomes even more complex and differs from the case of $M=3$, which we explain by an example in the Appendix. We leave the case of $M\geq 4$ for future research.}  Notably, the inhibiting property of the CDH robot allocation implies that the intrinsic matrix satisfies $\mathbf{I}_{ij}>1$ for $(i,j)\in\{(1,2),(2,3),(3,1)\}$.

\subsubsection{Construction of outcome interface}
\label{sssec:pi_oi}
Based on the elimination transformation, we first construct an \textit{outcome surface} that distinguishes between the win and loss of the players, and then design a new utility function of the game in \secref{sssec:u_for_heter}. The outcome interface $\pi_{\texttt{oi}}(\mathbf{x})=\pi_{\texttt{oi}}((u,v,w))=0$ defined on $\mathbb{R}^3$ is the continuous surface between any two of the three concurrent lines:
\begin{subequations}
    \begin{align}
    l_1&: u+v/\mathbf{I}_{12}=0, w=0,\label{l_1}\\
    l_2&: v+w/\mathbf{I}_{23}=0, u=0,\label{l_2}\\
    l_3&: w+u/\mathbf{I}_{31}=0, v=0.\label{l_3}
    \end{align}
    \label{eq:lines}
\end{subequations}
Then we formally define the piecewise linear function $\pi_{\texttt{oi}}(\mathbf{x})$ according to \cite[Definition 2.1]{ovchinnikov2000max}. 
\begin{definition}[\cite{ovchinnikov2000max}]
\itshape
    Let $\Gamma$ be a closed convex domain in $\mathbb{R}^d$. A function $\pi: \Gamma \rightarrow \mathbb{R}$ is said to be piecewise linear if there is a finite family $\mathcal{Q}$ of closed domains such that $\Gamma=\cup\mathcal{Q}$ and if $\pi$ is linear on every domain in $\mathcal{Q}$. A unique linear function $g$ on $\mathbb{R}^d$ which coincides with $\pi$ on a given $Q\in\mathcal{Q}$ is said to be a component of $\pi$. Let $\mathcal{H}$ denote the set of hyperplanes defined by $g_i(\mathbf{x})=g_j(\mathbf{x})$ for $i<j$ that have a nonempty intersection with the interior of $\Gamma$.
    \label{def:piecewiselinear}
\end{definition}

With the definition of the piecewise linear function, we utilize \cite[Theorem 4.1]{ovchinnikov2000max} to design our outcome interface, which defines the demarcation between the win and loss. 
\begin{theorem}[\cite{ovchinnikov2000max}]
\itshape
    Let $\pi$ be a piecewise linear function on $\Gamma=\mathbb{R}^3$ and $\{g_1,\cdots,g_n\}$ be the set of its distinct components. There exist a family $\{F_j\}_{j\in J}$ of subsets of $\{1,2,\cdots n\}$ such that 
    \begin{equation}
        \pi(\mathbf{x})=\max_{j\in J}\min_{i\in F_j}g_i(\mathbf{x}),\quad \forall \mathbf{x}\in \Gamma .
    \label{max-min}
    \end{equation}
    Conversely, for any family of distinct linear function $\{g_1,\cdots,g_n\}$, the above formula defines a piecewise linear function.
    \label{theorem:piecewise-linear}
\end{theorem}

The expression on the right side in \eqref{max-min} is a Max-Min (lattice) polynomial in the variables $g_i$ \cite{ovchinnikov2000max}. 

We write the $\pi_{\texttt{oi}}(\mathbf{x})$ in Max-Min term by Theorem \ref{theorem:piecewise-linear}. Let $\{g_1, g_2, g_3\}$ be the set of components of $\pi_{\texttt{oi}}(\mathbf{x})$ on $\mathbb{R}^3$ with
\begin{subequations}
\begin{align}
        g_1(\mathbf{x})&= [1,\ \mathbf{I}_{23}\mathbf{I}_{31},\ \mathbf{I}_{31}]\cdot \mathbf{x},\label{g_1}\\
        g_2(\mathbf{x})&= [\mathbf{I}_{12},\ 1,\ \mathbf{I}_{12}\mathbf{I}_{31}]\cdot \mathbf{x},\label{g_2}\\
        g_3(\mathbf{x})&= [\mathbf{I}_{12}\mathbf{I}_{23},\ \mathbf{I}_{23},\ 1]\cdot \mathbf{x}.\label{g_3}
\end{align} 
\label{eq:g}
\end{subequations}
where $g_1(\mathbf{x})=0$, $g_2(\mathbf{x})=0$, and $g_3(\mathbf{x})=0$ represent the planes constituted by the intersection of $l_2$ in \eqref{l_2} and $l_3$ in \eqref{l_3}, the intersection of $l_1$ in \eqref{l_1} and $l_3$ in \eqref{l_3}, and intersection of $l_1$ in \eqref{l_1} and $l_2$ in \eqref{l_2}, respectively. Let $J=\{1,2,3\}$, and $F_1=\{1,2\},F_2=\{1,3\},F_3=\{2,3\}$, then 
\begin{align}
\begin{aligned}
    \pi_{\texttt{oi}}(\mathbf{x})=&\max_{j\in J}\min_{i\in F_j}g_i(\mathbf{x}),\\
    =&\max\{\min\{g_1(\mathbf{x}),g_2(\mathbf{x})\},\ \min\{g_1(\mathbf{x}), g_3(\mathbf{x})\},\\
    &\min\{g_2(\mathbf{x}), g_3(\mathbf{x})\}\}.
\end{aligned}
\label{eq:pi}
\end{align}
Therefore the close form expression of outcome surface is 
\begin{align*}
    \pi_{\texttt{oi}}(\mathbf{x})=&\max\{\min\{g_1(\mathbf{x}),g_2(\mathbf{x})\},\ \min\{g_1(\mathbf{x}), g_3(\mathbf{x})\},\\
    &\min\{g_2(\mathbf{x}), g_3(\mathbf{x})\}\},\\=&\ 0.
\end{align*}
\figref{fig:illus_linearpiecewise} illustrates $g_i(\mathbf{x}),\mathcal{H},Q$ and $\pi_{\texttt{oi}}(\mathbf{x})$. 

\begin{figure}[!tbp]
\centering
    \subfloat[]{%
    \label{fig: Q}
    \includegraphics[width=0.8\columnwidth, height=1.5in]{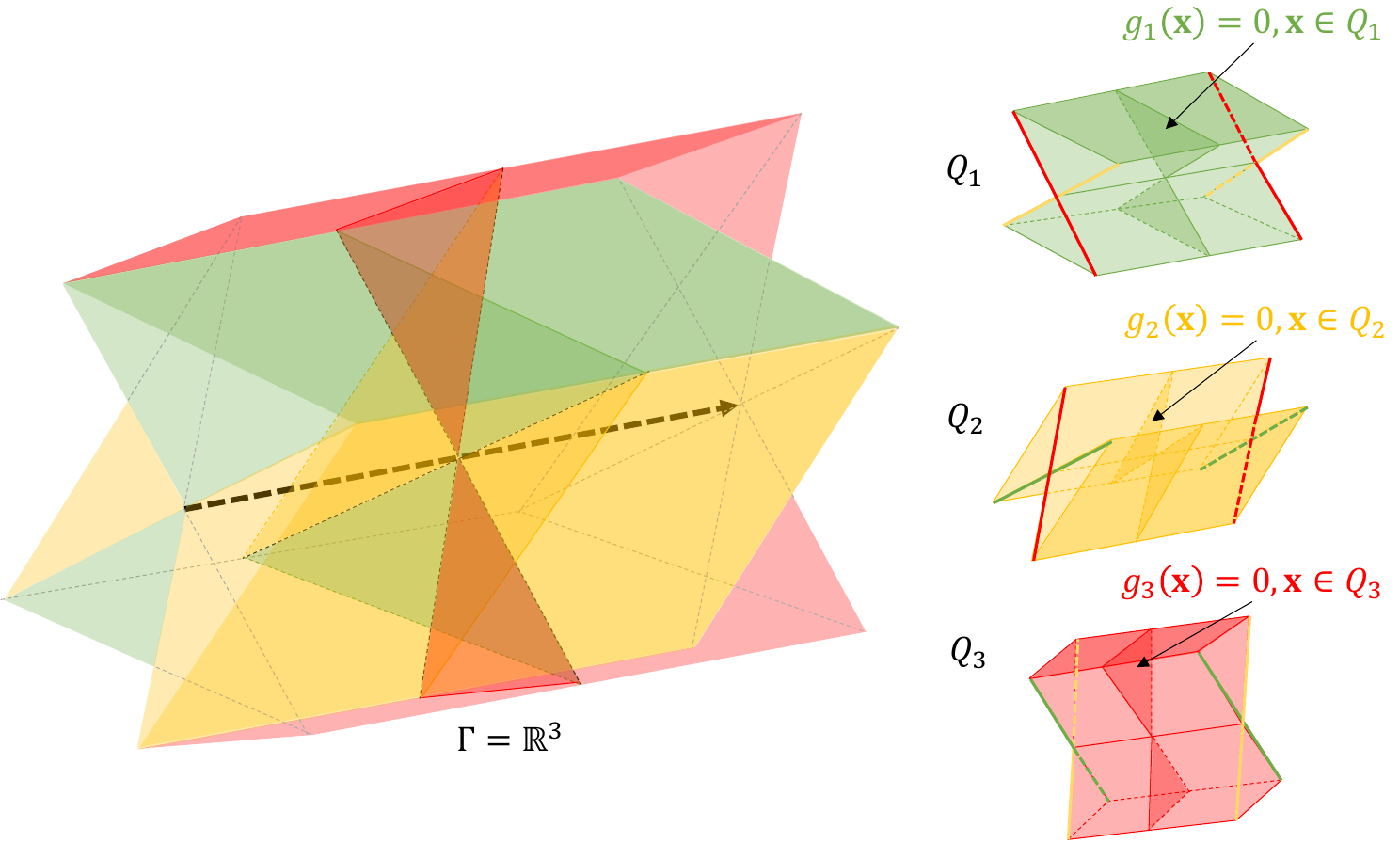}}%
    
    \subfloat[]{%
    \label{fig: PI}
    \includegraphics[width=0.5\columnwidth]{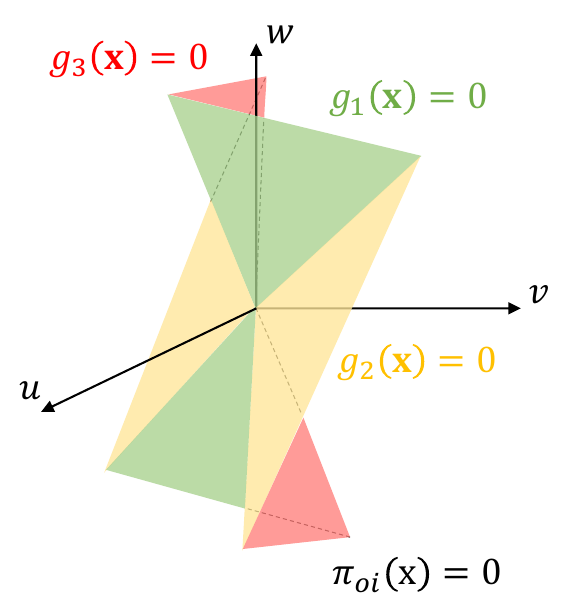}}%
    \caption{The top panel (a) illustrates the composition of the function's definition space \( \Gamma \), which is partitioned into three subspaces \( Q_1 \), \( Q_2 \), \( Q_3 \). In each subspace, the function \( \pi_{\texttt{oi}}(\mathbf{x}) \) is linear, that is, \( \pi_{\texttt{oi}}(\mathbf{x}) = g_i(\mathbf{x}) \) for \( \mathbf{x} \in Q_i \) ($Q_i$ extending indefinitely outward into space).  The plot on the right of (a) shows the positions where \( g_i(\mathbf{x}) = 0 \). 
    (b) displays the relative positioning of the surface \( \pi_{\texttt{oi}}(\mathbf{x}) = 0 \) in the Cartesian coordinate system. The red half-plane is $g_1$ in \eqref{g_1}, the green half-plane is $g_2$ in \eqref{g_2}, and the yellow half-plane is $g_3$ in \eqref{g_3}. These three half-planes together form the outcome interface, i.e., \(\pi_{\texttt{oi}}(\mathbf{x}) = 0 \). 
}   
    \label{fig:illus_linearpiecewise}
\end{figure}

\subsubsection{Construction of Utility Function}
\label{sssec:u_for_heter}
In this section, we first prove that the outcome surface is the demarcation surface between win and loss. Then we construct the utility function based on the outcome interface. Suppose Player 1 allocates $(u_1R_1, v_1R_2, w_1R_3)$ and Player 2 allocates $(u_2R_1, v_2R_2, w_2R_3)$ on a node, the robot distribution after subtraction is $\delta u = u_1-u_2$, $\delta v = v_1-v_2$, and $\delta w = w_1-w_2$.

\noindent \textbf{Winning condition.} Notably, Player 1 wins on a given node if and only if the remaining robot distribution after subtraction on the node $(\delta u, \delta v, \delta w)$ can be elimination transformed into $(\delta u^*, \delta v^*, \delta w^*)$, with $\delta u^*, \delta v^*, \delta w^* \geq 0$ and not all of them equal to $0$. Player 2 wins on a given node if and only if the remaining robot distribution after subtraction on the node $(\delta u, \delta v, \delta w)$ can be eliminated transformed into $(\delta u^*, \delta v^*, \delta w^*)$, with $\delta u^*, \delta v^*, \delta w^* \leq 0$ and not all of them equal to $0$. The game reaches a draw on a given node if and only if the remaining robot distribution after subtraction on the node $(\delta u, \delta v, \delta w)$ can be eliminated transformed into $(\delta u^*, \delta v^*, \delta w^*)$, with $\delta u^*, \delta v^*, \delta w^* = 0$.

Next, we show the completeness of the winning condition, i.e., 
\begin{theorem}
\itshape
    $\forall \mathbf{x} \in \mathbb{R}^3$~\footnote{In the appendix, we give an example demonstrating that when the dimension of $\mathbf{x}$ is larger than 3, $\mathbf{x}$ no longer belongs to one of these three conditions only.}, $\mathbf{x}$ belongs and only belongs to one of the three conditions. 
    \label{completeness}
\end{theorem}
Theorem \ref{completeness} tells that if $\mathbf{x}$ can be elimination transformed into $\mathbf{x}'$ with all entries of $\mathbf{x}'\geq 0$ and $\mathbf{x}'\neq 0$, then there is no other different elimination transformation that can transform it into $\mathbf{x}''$ that all entries of $\mathbf{x}''\leq 0$. This is obvious since the elimination transformation is a deterministic process in 3D. That is, for a vector $\mathbf{v}\in\mathbb{R}^3$,
\begin{itemize}
    \item if all entries of $\mathbf{v}$ are concordant, then by definition $\mathbf{v}$ can not be eliminated transformed, or $\mathbf{v}$ can only be eliminated transformed into itself.
    \item if not all entries of $\mathbf{v}$ are concordant, we sequentially search through pairs of entries in $\mathbf{v}$, specifically the combinations $(1, 2)$, $(1, 3)$, and $(2, 3)$. We select the first pair with different signs to operate an elimination transformation and continue this process until the value of one element in a pair reaches zero. This is a deterministic process and if the newly obtained vector $\mathbf{v}'$ can still undergo an elimination transformation, meaning not all its entries have the same sign. Then we continue this process until we achieve a vector where all entries are concordant. In 3D, this process needs at most two iterations to ensure that the resulting vector has entries of the same sign. Again, each step of this process is deterministic.
\end{itemize}
Theorem \ref{completeness} ensures that the defined winning condition divides space $\mathbb{R}^3$ into three mutually exclusive subspaces. Indeed, the surface $\pi_{\texttt{oi}}(\mathbf{x})=0$ is one of the three subspaces, more precisely, the \emph{tie game} subspace. $\pi_{\texttt{oi}}(\mathbf{x})>0$ corresponding to the Player 1 winning space and $\pi_{\texttt{oi}}(\mathbf{x})<0$ corresponding to the Player 2 winning space. That is,

\begin{theorem}
    At a node, Player 1 wins if and only if $\pi_{\texttt{oi}}(\delta u, \delta v, \delta w)>0$, Player 2 wins if and only if $\pi_{\texttt{oi}}(\delta u, \delta v, \delta w)<0$, and the game is tied if and only if $\pi_{\texttt{oi}}(\delta u, \delta v, \delta w)=0$. 
\end{theorem}

Without loss of generality, we prove the winning condition for Player 1. The winning condition for player 2 and the condition for a draw can be proved similarly. 
\begin{proof}
    First, we prove the sufficiency, i.e., if $\pi_{\texttt{oi}}(\delta u, \delta v, \delta w)>0$, then Player 1 wins. According to the winning condition, the winning of Player 1 is that $(\delta u,\delta v, \delta w)$ can be elimination transformed into $(\delta u^*,\delta v^*, \delta w^*)$, where $\delta u^*,\delta v^*, \delta w^* \geq 0$ and not all of them are $0$.

     If $\delta u,\delta v, \delta w \geq 0$ and not all of them are $0$, then $(\delta u^*,\delta v^*, \delta w^*)$ is exactly $(\delta u,\delta v, \delta w)$.
     If not all $\delta u,\delta v, \delta w \geq 0$, without loss of generality, we assume $\delta w<0$ and $\delta u, \delta v\geq 0$ and consider three cases below. 
     \begin{itemize}
         \item \textit{Case 1:} $\mathbf{x}=(\delta u, \delta v, \delta w)\in \mathbf{Q}_1$, i.e., $\pi_{\texttt{oi}}(\mathbf{x})=g_1(\mathbf{x})>0$. In this case, we first transform $\delta u R_1$ into $R_3$ until one of them becomes $0$. If $R_3$ is completely eliminated by $R_1$, or, 
         \begin{align}
             \frac{\delta u}{\mathbf{I}_{31}}+ \delta w\geq 0.
             \label{proof_0}
         \end{align}
         then through elimination transformation $(\delta u, \delta v, \delta w)$ is transformed into $(\delta u+\mathbf{I}_{31}\delta w,\delta v,0)$ with each entry no less than $0$. If $\frac{\delta u}{\mathbf{I}_{31}} + \delta w<0$, it means $\delta uR_1$ does not suffice to eliminate $R_3$, and thus we further transform $\delta v R_2$ into $R_3$. If
         \begin{align}
             \frac{\delta u}{\mathbf{I}_{31}}+\delta w + \mathbf{I}_{23}\delta v > 0.
             \label{proof_1}
         \end{align} 
         then $\delta v R_2$ suffices to eliminate remaining $R_3$, and through elimination transformation $(\delta u, \delta v, \delta w)$ is transformed into $(0, \delta v + \frac{\delta u}{\mathbf{I}_{31}\mathbf{I}_{23}}+\frac{\delta w}{\mathbf{I}_{23}}, 0)$ with each entry no less than $0$. Indeed, \eqref{proof_1} is valid since 
         \begin{align*}
             \begin{aligned}
                 \eqref{proof_1}&=\frac{1}{\mathbf{I}_{31}}(\delta u + \mathbf{I}_{31}\mathbf{I}_{23}\delta v + \mathbf{I}_{31}\delta w)\\
                 &=\frac{1}{\mathbf{I}_{31}}g_1(\mathbf{x})>0.
             \end{aligned}
         \end{align*} 
         Thus in this case $(\delta u, \delta v, \delta w)$ can always be elimination transformation into $(\delta u^*, \delta v^*, \delta w^*)$ with entries no less than $0$ and not all of them are $0$.
         
         \item \textit{Case 2:} $\mathbf{x}=(\delta u, \delta v, \delta w)\in \mathbf{Q}_3$, i.e., $\pi_{\texttt{oi}}(\mathbf{x})=g_3(\mathbf{x})>0$. We can still run the transformation process the same as for case 1, but need extra steps to prove \eqref{proof_1}. Note that \eqref{eq:pi} reveals an implication of $\pi_{\texttt{oi}}(\mathbf{x})$, in face $\pi_{\texttt{oi}}(\mathbf{x})$ represents the median of $g_1(\mathbf{x}),g_2(\mathbf{x})$ and $g_3(\mathbf{x})$ given $\mathbf{x}$. Therefore, with $\mathbf{x}\in\mathbf{Q}_3$, we obtain
         \begin{align}
             \min\{g_1(\mathbf{x}),g_2(\mathbf{x})\}\leq g_3(\mathbf{x}) \leq \max\{g_1(\mathbf{x}),g_2(\mathbf{x})\}.
             \label{proof_2}
         \end{align}
         Given that $\delta w<0$ and $\delta u, \delta v\geq 0$, it is clear that $g_3(\mathbf{x})\geq g_2(\mathbf{x})$. Thus \eqref{proof_2} becomes $g_2(\mathbf{x})\leq g_3(\mathbf{x})\leq g_1(\mathbf{x})$. Therefore in this case $g_1(\mathbf{x})\geq g_3(\mathbf{x})=\pi_{\texttt{oi}}(\mathbf{x})\geq 0$, which shows that \eqref{proof_1} is valid.
        
         \item \textit{Case 3:} $\mathbf{x}=(\delta u, \delta v, \delta w)\in \mathbf{Q}_2$, i.e., $\pi_{\texttt{oi}}(\mathbf{x})=g_2(\mathbf{x})>0$. Based on the proofs of case 1 and case 2, we know that in this case $g_2(\mathbf{x})$ is the median of $g_1(\mathbf{x}),g_2(\mathbf{x}), g_3(\mathbf{x})$. $g_3(\mathbf{x})\geq g_2(\mathbf{x})$ is valid as long as $\delta w<0, \delta u, \delta v\geq 0$ holds. Thus we have $g_1(\mathbf{x})\leq g_2(\mathbf{x})\leq g_3(\mathbf{x})$.  With $g_2(\mathbf{x})\geq g_1(\mathbf{x})$, \eqref{proof_0} always holds, which means $\delta u R_1$ always suffices to eliminate $\delta w R_3$. 
         Indeed, $g_2(\mathbf{x})\geq g_1(\mathbf{x})$ can be rewritten as
         \begin{align}
            \mathbf{I}_{31}(\mathbf{I}_{12}-1)(\frac{\delta u}{\mathbf{I}_{31}}+\delta w) \geq (\mathbf{I}_{23}\mathbf{I}_{31}-1)\delta v.
            \label{proof_4}
         \end{align}
         Given $(\mathbf{I}_{23}\mathbf{I}_{31}-1)\delta v\geq 0$, \eqref{proof_4} leads to \eqref{proof_0}. This means $(\delta u,\delta v, \delta w)$ can always be eliminated into $(\delta u+\mathbf{I}_{31}\delta w,\delta v,0)$.
     \end{itemize}
    Therefore, the sufficiency is proven.
    
    Then we prove the necessity, i.e., if Player 1 wins (or equivalently, $(\delta u,\delta v, \delta w)$ can be elimination transformed into $(\delta u^*,\delta v^*, \delta w^*)$ with $\delta u^*,\delta v^*, \delta w^* \geq 0$ and not all of them being $0$), then $\pi_{\texttt{oi}}(\delta u, \delta v, \delta w)>0$. We consider two cases. 

    \begin{itemize}
        \item \textit{Case 1:} $\delta u,\delta v, \delta w \geq 0$ and not all of them are $0$, then $(\delta u^*,\delta v^*, \delta w^*)$ is exactly $(\delta u,\delta v, \delta w)$.  Obviously, $\pi_{\texttt{oi}}(\delta u, \delta v, \delta w)>0$, since the coefficients of piecewise linear function $\pi_{\texttt{oi}}(\cdot)$ are positive.
        \item \textit{Case 2:} Not all $\delta u,\delta v, \delta w \geq 0$, assuming $\delta w < 0$. Upon performing a elimination transformation on $(\delta u,\delta v, \delta w)$, either \eqref{proof_0} or \eqref{proof_1} must hold since the result invariably ensures that all entries are non-negative. If \eqref{proof_0} holds, then $g_1(\mathbf{x}),g_2(\mathbf{x}),g_3(\mathbf{x}) > 0$, leading to $\pi_{\texttt{oi}}(\mathbf{x})>0$. If \eqref{proof_0} does not hold, then \eqref{proof_1} must hold and $\pi_{\texttt{oi}}(\mathbf{x})\neq g_2(\mathbf{x})$ since we have shown that if $\pi_{\texttt{oi}}(\mathbf{x})= g_2(\mathbf{x})$, then \eqref{proof_0} must hold in the proof of necessity. Therefore, either $\pi_{\texttt{oi}}(\mathbf{x})=g_1(\mathbf{x})$ or $\pi_{\texttt{oi}}(\mathbf{x})=g_3(\mathbf{x})$. If $\pi_{\texttt{oi}}(\mathbf{x})=g_1(\mathbf{x})$ then \eqref{proof_1} directly leads to $\pi_{\texttt{oi}}(\mathbf{x})>0$. If $\pi_{\texttt{oi}}(\mathbf{x})=g_3(\mathbf{x})$, from $\frac{\delta u}{\mathbf{I}_{31}}+\delta w <0$ and \eqref{proof_1}, we have 
        \begin{align}
            \delta v \geq -(\frac{\delta u}{\mathbf{I}_{31}\mathbf{I}_{23}}+\frac{\delta w}{\mathbf{I}_{23}}).
            \label{proof_5}
        \end{align}
        Substituting \eqref{proof_5} into $g_3(\mathbf{x})$, we have
        \begin{align*}
            \begin{aligned}
                g_3(\mathbf{x})&=\mathbf{I}_{12}\mathbf{I}_{23}\delta u + \mathbf{I}_{23}\delta v + \delta w\\
                &\geq (\mathbf{I}_{12}\mathbf{I}_{23}-1)\delta u\\
                &>0.
            \end{aligned}
        \end{align*}
        Then we have $\pi_{\texttt{oi}}(\mathbf{x})>0$. 
    \end{itemize}  
    Therefore the necessity is also proved.
\end{proof}
    
We use an example to further clarify the property of the outcome interface. Suppose the remaining robot distribution after subtraction on a node is $(\delta u, \delta v, \delta w)=(4,2,-7)$. We evaluate this outcome from two views. First, intuitively, Player 2 eliminates all Player 1's $2R_2$ by $-4R_3$ since $R_2=2R_3$, and eliminates all Player 1's $4R_1$ by $-2R_3$ since $R_3=2R_1$, and in the end, Player 2 still has $1R_3$ while Player 1 remains nothing. Thus Player 2 wins on this node. In this process, we first convert $R_2$ into $R_3$, and then convert $R_1$ into $R_3$, and finally find out that after the transformation there is still remaining $R_3$ on Player 2's side. The second view is through the outcome surface $\pi_{\texttt{oi}}(\mathbf{x})$. Plugging $\mathbf{x}=(\delta u, \delta v, \delta w)=(4,2,-7)$ into \eqref{eq:pi}, we have $\pi_{\texttt{oi}}((4,2,-7))=-2<0$, indicating Player 2 wins. Indeed the calculation of $\pi_{\texttt{oi}}((4,2,-7))$ includes three steps. We first compute $g_1((4,2,-7))$, i.e., calculating the outcome of converting all types of robots into the $R_1$, i.e., is $-18 R_1$. Second, we compute $g_2((4,2,-7))$, i.e., calculating the outcome of converting all types of robots into the $R_2$, which is $13R_2$. Thirdly, we compute $g_3((4,2,-7))$, i.e., calculating the outcome of converting all types of robots into the $R_3$, which is $-2R_3$. Finally, we choose the middle value to represent the outcome according to the definition of $\pi_{\texttt{oi}}$. This is consistent with the intuition. The reason that the value from the outcome surface $-2$ is different from that of the intuition $-1$ is that the plane function $g_1(\cdot)$, component of $\pi_{\texttt{oi}}(\cdot)$, is $\mathbf{I}_{13}$ multiplying the converted term $(\delta u/\mathbf{I}_{31}+\mathbf{I}_{23}\delta v+\delta w)$.

The utility function $u_{\texttt{CDH}}(\cdot)$ is then computed based on the outcome interface. To ensure function continuity, we assume that Player 1 is only considered to completely win when \( \pi_{\texttt{oi}}(\mathbf{x}) > C \) and the utility for Player 1 is 1. Otherwise, if \( 0 < \pi_{\texttt{oi}}(\mathbf{x}) < C \), Player 1 has an incomplete victory, with their utility being a number that is between 0 and 1. The same applies to Player 2. Note that the direction of the surface's movement from $\pi_{\texttt{oi}}(\mathbf{x}) = 0$ to $ \pi_{\texttt{oi}}(\mathbf{x}) = \epsilon $ is denoted by the black dashed lines in \figref{fig:illus_linearpiecewise} ($\epsilon$ is a constant).
The illustration of the constructed utility function is shown in \figref{fig:b}.

\begin{figure}
    \centering
    \includegraphics[width=3in]{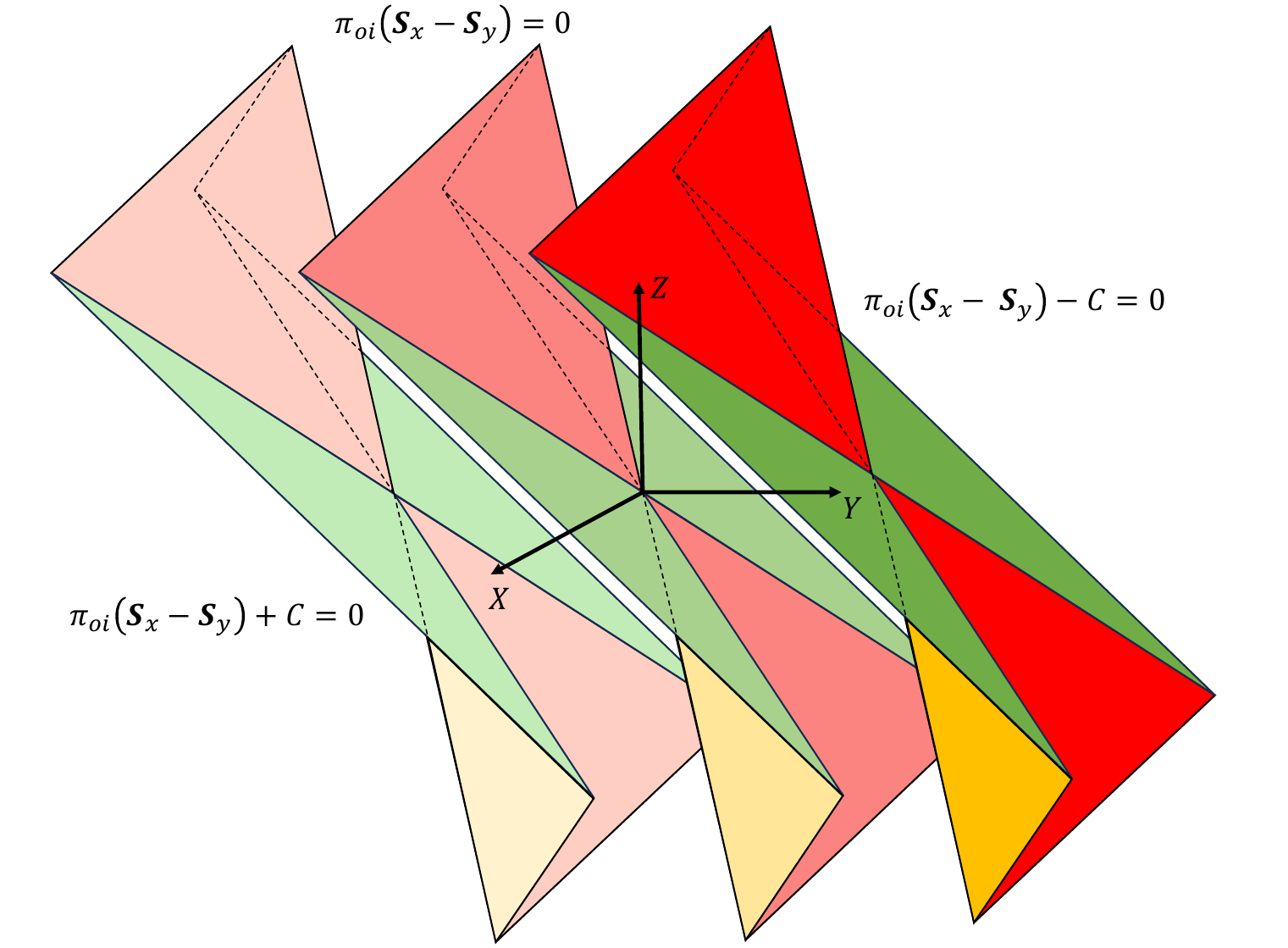}
    \caption{Illustration of utility function $u_{\texttt{CDH}}$. The surface $\pi_{\texttt{oi}}(\mathbf{S}_x-\mathbf{S}_y)$ in the middle is the outcome interface in \figref{fig: PI}.}
    \label{fig:b}
\end{figure}
\begin{align}
    \label{eq:heter-utility}
     u_{\texttt{CDH}}(\mathbf{S}_x, \mathbf{S}_y) = \sum_{i=1}^{N}{\pi(\mathbf{S}_{x,i} - \mathbf{S}_{y,i})}.
\end{align}
\begin{align*}
\text{where} \quad \pi(\mathbf{x}) = \left\{
        \begin{aligned}
        -1&, &\pi_{\texttt{oi}}(\mathbf{x})&\leq-C;\\
        \pi_{\texttt{oi}}(\mathbf{x})/C&, &\pi_{\texttt{oi}}(\mathbf{x})&\in [-C, C];\\
        1&, &\pi_{\texttt{oi}}(\mathbf{x})&\geq C.
        \end{aligned}
        \right.
\end{align*}
$C$ is the threshold for win or loss. The remaining robot distribution after subtraction on $i$-th node $\mathbf{S}_{x,i} - \mathbf{S}_{y,i}=(\delta u, \delta v, \delta w)$ is situated on the surface $\pi_{\texttt{oi}}$. If $\pi_{\texttt{oi}}>C$, then $u_{\texttt{CDH}}(\mathbf{S}_{x,i} -\mathbf{S}_{y,i})=1$, indicating an absolute win for Player 1. Conversely, if $\pi_{\texttt{oi}}<-C$, then $u_{\texttt{CDH}}(\mathbf{S}_{x,i} - \mathbf{S}_{y,i})=-1$, indicating an absolute win for Player 2. If $-C < \pi_{\texttt{oi}} < C$, then $u_{\texttt{CDH}}(\mathbf{S}_{x,i} - \mathbf{S}_{y,i})=\pi_{\texttt{oi}}/C$, meaning that one side achieves winning to a certain extent. This way of constructing the utility function ensures its continuity~\cite{adam2021double}, consequently facilitating the calculation of the best response strategy (\algref{alg:doa}, \linref{alg1.2}).

\subsection{DOA with CDH robots}
\label{ssec:doa_cdh}
We next introduce the steps of solving the optimization problem of \eqref{eq:bst_st} in \algref{alg:doa}, \linref{alg1.2} with the constructed utility function $u_{\texttt{CDH}}(\cdot)$. The main idea is to rewrite the piecewise linear function into linear inequalities and transform the problem into a linear programming problem. 

In \eqref{eq:pi}, $\pi_{\texttt{oi}}(\mathbf{x})$ has the Max-Min representation. We further rewrite it by linear inequalities. First, we rewrite it as
\begin{align}
\begin{aligned}
    \pi_{\texttt{oi}}(\mathbf{x}) = &\ g_1(\mathbf{x}) + g_2(\mathbf{x}) + g_3(\mathbf{x})-\min\{g_1(\mathbf{x}), g_2(\mathbf{x}), g_3(\mathbf{x})\}\\
    &- \max\{g_1(\mathbf{x}), g_2(\mathbf{x}), g_3(\mathbf{x})\}\\
    = &\ g_1(\mathbf{x}) + \max\{g_2(\mathbf{x})-g_1(\mathbf{x}), g_2(\mathbf{x})-g_3(\mathbf{x}),0\}\\
    &- \max\{g_1(\mathbf{x})-g_3(\mathbf{x}), g_2(\mathbf{x})-g_3(\mathbf{x}),0\}.
\end{aligned}
\label{eq:pi_updated}
\end{align}
Note that \eqref{eq:pi_updated} expresses $\pi_{\texttt{oi}}(\mathbf{x})$ in the form of linear combination of $\max\{f_1(\mathbf{x}),f_2(\mathbf{x}),0\}$. To rewrite it into linear inequalities, \lemmaref{corollary:ieq_new} is applied. To begin with, we first introduce \cite[Lemma C.1]{adam2021double} to write the functions in the form of $F(\mathbf{x})=\max\{f(\mathbf{x}),0\}$ into inequalities form.
\begin{lemma}[\cite{adam2021double}]
\itshape
    Let $f(\mathbf{x})$ be a function of $\mathbb{R}^d \rightarrow \mathbb{R}$, $U>0$, and $F(\mathbf{x})=\max\{f(\mathbf{x}),0\}$. For every $\mathbf{x}$ such that $f(\mathbf{x}) \in [-U,U]$ there is a unique $s\in \mathbb{R}$ and a possible non-unique $z\in\{0,1\}$ solving the system
    \begin{align*}
        \begin{aligned}
        s&\geq 0,             &s&\leq Uz,\\
        s&\geq f(\mathbf{x}), &s&\leq f(\mathbf{x})+U(1-z).
        \end{aligned}
    \end{align*}
    and it holds $F(\mathbf{x})=s$.
\label{corollary:ieq}
\end{lemma}
Based on \lemmaref{corollary:ieq}, one has
\begin{lemma}
\itshape
    Let $f_1(\mathbf{x}), f_2(\mathbf{x})$ be a function of $\mathbb{R}^d \rightarrow \mathbb{R}$, $U>0$, and $F(\mathbf{x})=\max\{f_1(\mathbf{x}),f_2(\mathbf{x}),0\}$. For every $\mathbf{x}$ such that $f_1(\mathbf{x}), f_2(\mathbf{x}), f_1(\mathbf{x})-f_2(\mathbf{x}) \in [-U,U]$ there is a unique $s,t\in \mathbb{R}$ and a possible non-unique $z_1, z_2\in\{0,1\}$ solving the system
    \begin{align}
        \begin{aligned}
        s&\geq 0,             &s&\leq Uz_1,\\
        s&\geq f_1(\mathbf{x}), &s&\leq f_1(\mathbf{x})+U(1-z_1),\\
        t&\geq f_2(\mathbf{x}),             &s&\leq f_2(\mathbf{x})+Uz_2,\\
        t&\geq s, &t&\leq s+U(1-z_2).\\
        \end{aligned}
        \label{ieq:cor1}
    \end{align}
    and it holds $F(\mathbf{x})=t$. 
\label{corollary:ieq_new}
\end{lemma}
\begin{proof}
$\forall \mathbf{x}$, s.t. $f_1(\mathbf{x}), f_2(\mathbf{x}), f_1(\mathbf{x})-f_2(\mathbf{x}) \in[-U,U]$, there is 
\begin{align*}
    \max\{f_1(\mathbf{x}), f_2(\mathbf{x}), 0\}&=\max\{\max\{f_1(\mathbf{x}),0\}, f_2(\mathbf{x})\}\\
    &=\max\{s, f_2(\mathbf{x})\}\\
    &=\max\{s-f_2(\mathbf{x}),0\}+f_2(\mathbf{x})\\
    &=t'+f_2(\mathbf{x}).
\end{align*}
where $s$ is the solution of the inequality system
\begin{align*}
    s&\geq 0,             &s&\leq Uz_1,\\
    s&\geq f_1(\mathbf{x}), &s&\leq f_1(\mathbf{x})+U(1-z_1).
\end{align*}
and $s=\max\{f_1(\mathbf{x}),0\}$, and $t'$ is the solution of the inequality system
\begin{align*}
    t'&\geq 0,             &t'&\leq Uz_2,\\
    t'&\geq s-f_2(\mathbf{x}), &t'&\leq s-f_2(\mathbf{x})+U(1-z_2).
\end{align*}
and $t'=\max\{s-f_2(\mathbf{x}),0\}$, $z_1, z_2\in\{0,1\}$, according to \lemmaref{corollary:ieq}.

Denote $\max\{f_1(\mathbf{x}), f_2(\mathbf{x}), 0\}$ as $t$, following $t=t'+f_2(\mathbf{x})$. Substituting $t$ into the linear inequalities above, we have
\begin{align*}
    t&\geq f_2(\mathbf{x}),             &t&\leq f_2(\mathbf{x})+Uz_2,\\
    t&\geq s, &t&\leq s+U(1-z_2).
\end{align*}
\end{proof}
Therefore we derive the linear inequalities equivalent to \eqref{eq:pi_updated} by replacing the $\max\{\cdot\}$ operation with \eqref{ieq:cor1}. 

With \lemmaref{corollary:ieq} and \lemmaref{corollary:ieq_new}, the optimization problem of \eqref{eq:bst_st} can be transformed into mixed integer linear programming (MILP), which can be solved by the optimization solver \texttt{gurobi} \cite{gurobi}. With utility function as $u_\texttt{CDH}(\cdot)$, the optimization problem \eqref{eq:bst_st} can be reformulated as follows:

Given Player 2's $j$-th step mixed strategy 
$\mathbf{\Delta}_Y^{j*}\sim \text{Multinomial}(P_y)$,  
\begin{align}
\begin{aligned}
\max_{\mathbf{S}_x \in \mathcal{X}(\mathbf{d}_x)}\sum_{i=1}^KP_y(\mathbf{S}_y^i)\sum_{j=1}^N\pi(\mathbf{S}_{x,j}-\mathbf{S}_{y,j}^i).
\end{aligned} 
\label{eq:linearization}
\end{align}
According to \cite[Appendix C]{adam2021double}, for any $i\in\{1,2,\cdots,K\}$, $j\in\{1,2,\cdots,N\}$ and a constant $C$, there is
\begin{align*}
\begin{aligned}
    &\pi(\mathbf{S}_{x,j}-\mathbf{S}_{y,j}^i)\\
    =&\max\{\frac{1}{C}(\pi_{\texttt{oi}}(\mathbf{S}_{x,j}-\mathbf{S}_{y,j}^i)+C),0\}\\
    &-\max\{\frac{1}{C}(\pi_{\texttt{oi}}(\mathbf{S}_{x,j}-\mathbf{S}_{y,j}^i)-C),0\}-1.
\end{aligned}
\end{align*}
or, for simplicity, we can denote $\pi(\mathbf{S}_{x,j}-\mathbf{S}_{y,j}^i)$ as $s_{ij}-t_{ij}-1$ where $s_{ij}=\max\{\frac{1}{C}(\pi_{\texttt{oi}}(\mathbf{S}_{x,j}-\mathbf{S}_{y,j}^i)+C),0\}$ and $t_{ij}=\max\{\frac{1}{C}(\pi_{\texttt{oi}}(\mathbf{S}_{x,j}-\mathbf{S}_{y,j}^i)-C),0\}$. With \lemmaref{corollary:ieq}, we obtain $s_{ij}$ and $t_{ij}$ in the form of linear inequalities system as follows:
\begin{align*}
    s_{ij}&\geq 0,\\
    s_{ij}&\leq U_{ij}^sz_{ij},\\
    s_{ij}&\leq \frac{1}{C}(\pi_{\texttt{oi}}(\mathbf{S}_{x,j}-\mathbf{S}_{y,j}^i)+C)+U_{ij}^s(1-z_{ij}),\\
    s_{ij}&\geq \frac{1}{C}(\pi_{\texttt{oi}}(\mathbf{S}_{x,j}-\mathbf{S}_{y,j}^i)+C),\\ 
    t_{ij}&\geq 0,\\
    t_{ij}&\leq U_{ij}^tw_{ij},\\
    t_{ij}&\leq \frac{1}{C}(\pi_{\texttt{oi}}(\mathbf{S}_{x,j}-\mathbf{S}_{y,j}^i)-C)+U_{ij}^t(1-w_{ij}),\\
    t_{ij}&\geq \frac{1}{C}(\pi_{\texttt{oi}}(\mathbf{S}_{x,j}-\mathbf{S}_{y,j}^i)-C).
\end{align*}
where $z_{ij},w_{ij}\in\{0,1\}$ and $\frac{1}{C}(\pi_{\texttt{oi}}(\mathbf{S}_{x,j}-\mathbf{S}_{y,j}^i)+C)\in[-U_{ij}^s,U_{ij}^s]$, $\frac{1}{C}(\pi_{\texttt{oi}}(\mathbf{S}_{x,j}-\mathbf{S}_{y,j}^i)-C)\in[-U_{ij}^t,U_{ij}^t]$, $\forall \mathbf{S}_{x}\in \mathcal{X}(\mathbf{d}_x)$. We can replace the $\pi_{\texttt{oi}}(\cdot)$ in the linear inequalities as discussed above. Specifically, according to \eqref{eq:pi_updated}, one has
\begin{align*}
    \begin{aligned}
        \pi_{\texttt{oi}}(\mathbf{S}_{x,j}-\mathbf{S}_{y,j}^i)=g_1(\mathbf{S}_{x,j}-\mathbf{S}_{y,j}^i)+p_{ij}-q_{ij}.
    \end{aligned}
\end{align*}
where $p_{ij}=\max\{g_{21}(\mathbf{S}_{x,j}-\mathbf{S}_{y,j}^i),\; g_{23}(\mathbf{S}_{x,j}-\mathbf{S}_{y,j}^i),\; 0\}$ and $q_{ij}=\max\{g_{13}(\mathbf{S}_{x,j}-\mathbf{S}_{y,j}^i),\; g_{23}(\mathbf{S}_{x,j}-\mathbf{S}_{y,j}^i),\; 0\}$, $g_{21}(\cdot)=g_2(\cdot)-g_1(\cdot)$, $g_{13}(\cdot)=g_1(\cdot)-g_3(\cdot)$, $g_{23}(\cdot)=g_2(\cdot)-g_3(\cdot)$ for simplicity. Moreover, according to \lemmaref{corollary:ieq_new} we can write $p_{ij}$ and $q_{ij}$ into the form of linear inequalities. Finally, the optimization problem in \eqref{eq:linearization} is reformulated into a mixed integer linear optimization problem as follows.

Given Player 2's $j$-th step mixed strategy $\mathbf{\Delta}_Y^{j*}\sim \text{Multinomial}(P_y)$,
\begin{align*}
    \max_{\mathbf{S}_x \in \mathcal{X}(\mathbf{d}_x)}&\sum_{i=1}^KP_y(\mathbf{S}_y^i)\sum_{j=1}^N(s_{ij}-t_{ij}-1)\\
    s.t.\quad &s_{ij}\geq 0,\quad s_{ij}\leq U_{ij}^sz_{ij},\\
    &s_{ij}\leq \frac{1}{C}(g_1(\mathbf{S}_{x,j}-\mathbf{S}_{y,j}^i)+p_{ij}-q_{ij}+C)+\\
    &\qquad\ \; U_{ij}^s(1-z_{ij}),\\
    &s_{ij}\geq \frac{1}{C}(g_1(\mathbf{S}_{x,j}-\mathbf{S}_{y,j}^i)+p_{ij}-q_{ij}+C),\\ 
    &t_{ij}\geq 0,\quad t_{ij}\leq U_{ij}^tw_{ij},\\
    &t_{ij}\leq \frac{1}{C}(g_1(\mathbf{S}_{x,j}-\mathbf{S}_{y,j}^i)+p_{ij}-q_{ij}-C)+\\
    &\qquad\ \; U_{ij}^t(1-w_{ij}),\\
    &t_{ij}\geq \frac{1}{C}(g_1(\mathbf{S}_{x,j}-\mathbf{S}_{y,j}^i)+p_{ij}-q_{ij}-C),\\
    &\delta^p_{ij}\geq 0,\quad \delta^p_{ij}\leq U^{\delta^p}_{ij}z^{\delta^p}_{ij},\\
    &\delta^p_{ij}\geq g_{21}(\mathbf{S}_{x,j}-\mathbf{S}_{y,j}^i),\\
    &\delta^p_{ij}\leq g_{21}(\mathbf{S}_{x,j}-\mathbf{S}_{y,j}^i)+U^{\delta^p}_{ij}(1-z^{\delta^p}_{ij}),\\
    &p_{ij}\geq \delta^p_{ij},\quad p_{ij}\leq \delta^p_{ij}+U^{pq}_{ij}(1-z^p_{ij}),\\
    &p_{ij}\geq g_{23}(\mathbf{S}_{x,j}-\mathbf{S}_{y,j}^i),\\
    &p_{ij}\leq g_{23}(\mathbf{S}_{x,j}-\mathbf{S}_{y,j}^i)+U^{pq}_{ij}z^p_{ij},\\
    &\delta^q_{ij}\geq 0,\quad \delta^q_{ij}\leq U^{\delta^q}_{ij}z^{\delta^q}_{ij},\\
    &\delta^q_{ij}\geq g_{13}(\mathbf{S}_{x,j}-\mathbf{S}_{y,j}^i),\\
    &\delta^q_{ij}\leq g_{13}(\mathbf{S}_{x,j}-\mathbf{S}_{y,j}^i)+U^{\delta^q}_{ij}(1-z^{\delta^q}_{ij}),\\
    &q_{ij}\geq \delta^q_{ij},\quad q_{ij}\leq \delta^q_{ij}+U^{pq}_{ij}(1-z^q_{ij}),\\
    &q_{ij}\geq g_{23}(\mathbf{S}_{x,j}-\mathbf{S}_{y,j}^i),\\
    &q_{ij}\leq g_{23}(\mathbf{S}_{x,j}-\mathbf{S}_{y,j}^i)+U^{pq}_{ij}z^q_{ij}.
\end{align*}
with
\begin{align*}
    s_{ij}, t_{ij}, p_{ij}, q_{ij}&\in \mathbb{R},\\
    z_{ij}, w_{ij}, z^{\delta^p}_{ij},z^{\delta^q}_{ij},z^p_{ij},z^q_{ij}&\in \{0,1\},\\
    \frac{1}{C}(\pi_{\texttt{oi}}(\mathbf{S}_{x,j}-\mathbf{S}_{y,j}^i)+C)&\in[-U_{ij}^s,U_{ij}^s],\\
    \frac{1}{C}(\pi_{\texttt{oi}}(\mathbf{S}_{x,j}-\mathbf{S}_{y,j}^i)-C)&\in[-U_{ij}^t,U_{ij}^t],\\
    g_{21}(\mathbf{S}_{x,j}-\mathbf{S}_{y,j}^i)&\in[-U_{ij}^{\delta^p},U_{ij}^{\delta^p}],\\
    g_{13}(\mathbf{S}_{x,j}-\mathbf{S}_{y,j}^i)&\in[-U_{ij}^{\delta^q},U_{ij}^{\delta^q}],\\
    g_{23}(\mathbf{S}_{x,j}-\mathbf{S}_{y,j}^i)&\in[-U_{ij}^{pq},U_{ij}^{pq}].
\end{align*}
Note that $g_{21}(\cdot), g_{13}(\cdot), g_{23}(\cdot)$ are linear functions and $\mathbf{S}_{x,j}-\mathbf{S}_{y,j}^i$ is bounded. Thus $g_{21}(\mathbf{S}_{x,j}-\mathbf{S}_{y,j}^i), g_{13}(\mathbf{S}_{x,j}-\mathbf{S}_{y,j}^i), g_{23}(\mathbf{S}_{x,j}-\mathbf{S}_{y,j}^i)$ are all bounded. $U_{ij}^{\delta^p}, U_{ij}^{\delta q}, U_{ij}^{pq}$ are the large numbers to include their bounds. Since $\pi_{\texttt{oi}}(\mathbf{S}_{x,j}-\mathbf{S}_{y,j}^i)$ is piecewise linear, both $\frac{1}{C}(\pi_{\texttt{oi}}(\mathbf{S}_{x,j}-\mathbf{S}_{y,j}^i)+C)$ and $\frac{1}{C}(\pi_{\texttt{oi}}(\mathbf{S}_{x,j}-\mathbf{S}_{y,j}^i)-C)$ are bounded and $U_{ij}^s, U_{ij}^t$ are large number to include their bounds. Besides, since $g_i(\cdot)$ are linear functions, thus all $g_i(\mathbf{S}_{x,j}-\mathbf{S}_{y,j}^i)$ above can be written as $g_i(\mathbf{S}_{x,j})-g_i(\mathbf{S}_{y,j}^i)$.
\section{Numerical Evaluation} 
\label{sec: evaluation}
\begin{figure}[!tbp]
\centering
    \includegraphics[width=1\columnwidth]{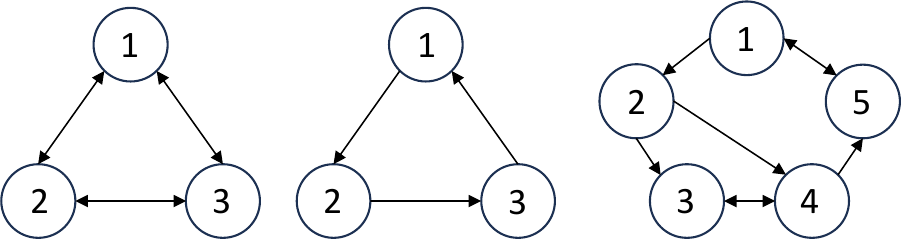}
    \caption {Three different graphs: $G_1$ (left), $G_2$ (middle), $G_3$ (right).}
    \label{fig: paper_3cases}
\end{figure}

\begin{figure}[!tbp]
\centering
    \includegraphics[height=2in]{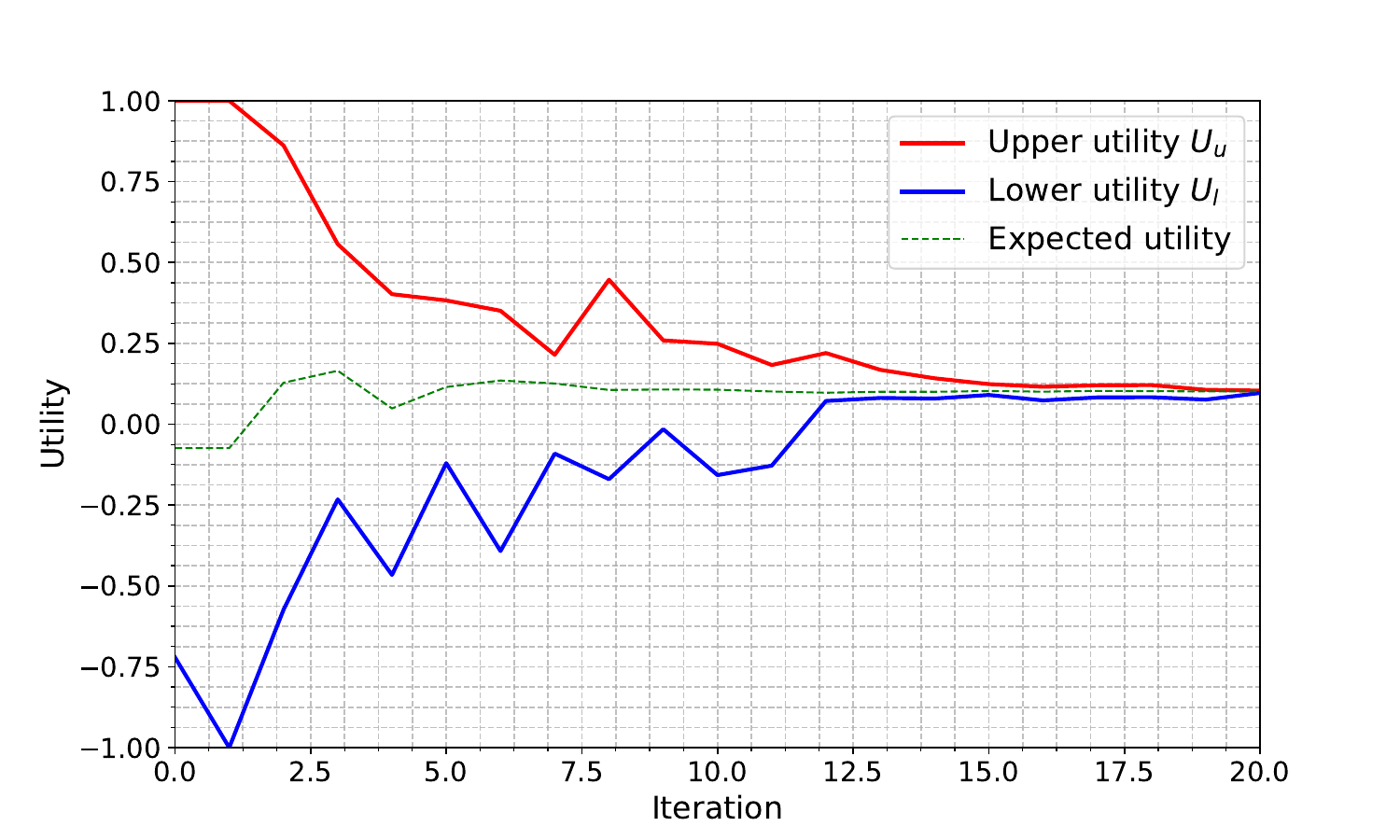}
    \caption{DOA achieves the equilibrium in homogeneous robot allocation on $G_2$. The blue curve shows the upper utility of the game $U_u$ and the red curve shows the lower utility of the game $U_l$ (\algref{alg:doa}). The green curve shows the expected utility of the game in each iteration. Runtime is 4.243s.} 
    \label{fig:homo_convergence}
\end{figure}
\begin{figure}[!tbp]
\centering
    \includegraphics[height=2in]{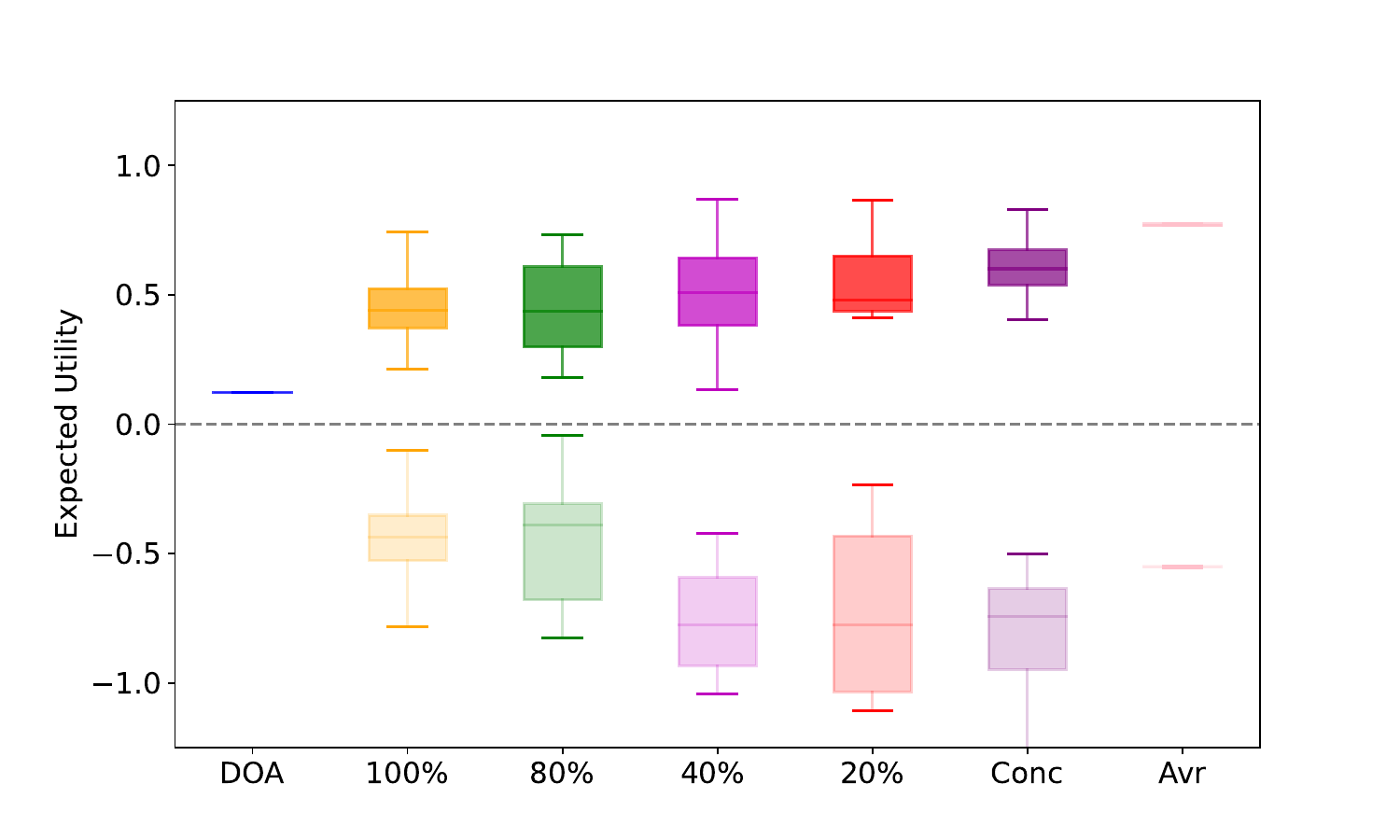}
    \caption{Comparison of the expected utilities by DOA and other baselines for homogeneous robot allocation on $G_2$.}
    \label{fig: evaluation_homo}
\end{figure}
\begin{figure*}[!tbp]
    \centering
    \includegraphics[width=1.98\columnwidth]{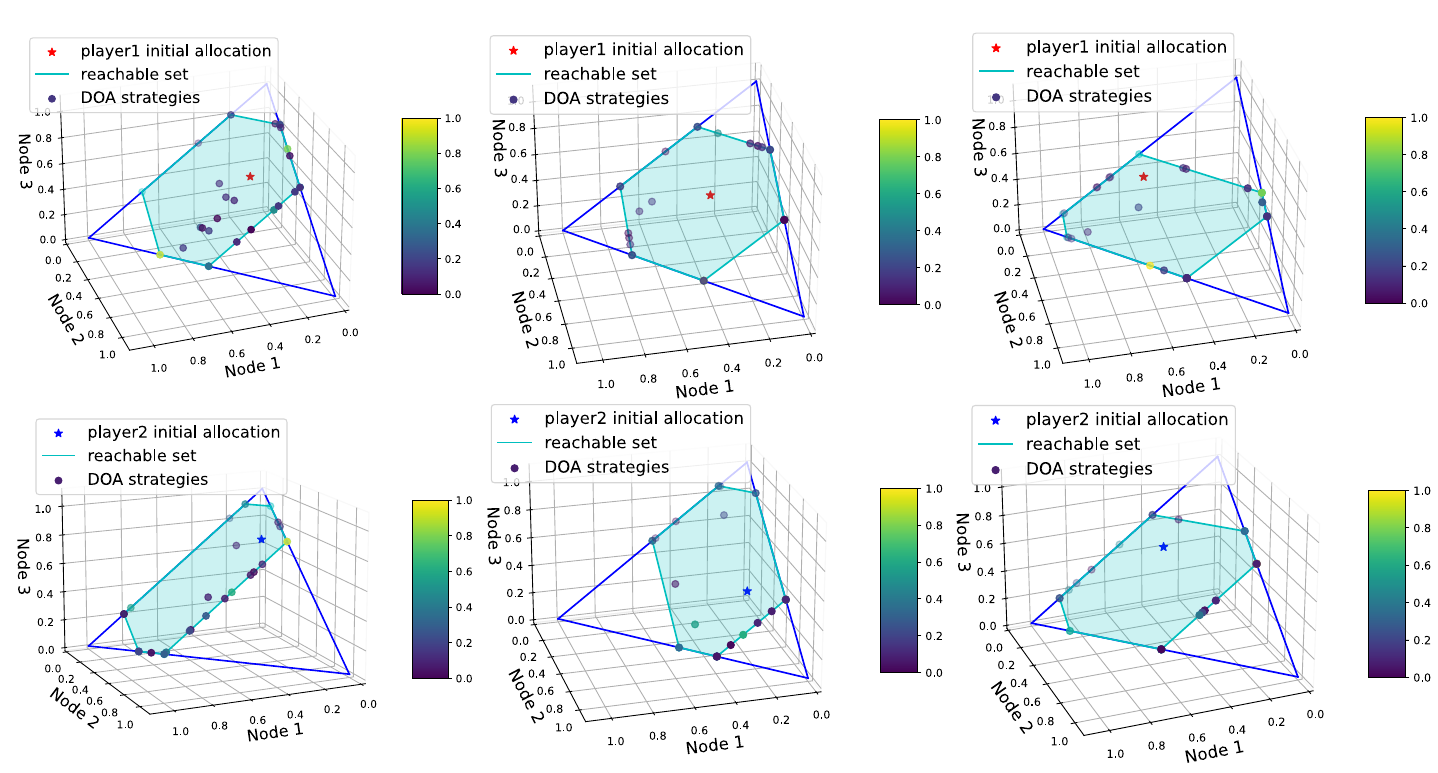}
    \caption{Illustration of the reachable sets and mixed strategies calculated by DOA. The top three subfigures show the mixed strategies of Player 1 and the bottom three subfigures show that of Player 2. In each subfigure, the light blue region is the next-step reachable set for the initial allocation (represented by the red and blue star for Player 1 and Player 2, respectively). The colorful dots on the reachable set are the mixed strategies for two players, each dot representing one pure strategy. The distribution of the mixed strategies is represented by the heatmap on the right-hand side of each subfigure. The heatmap, transitioning from blue to yellow, represents the relative probability $\frac{p-p_{\texttt{min}}}{p_{\texttt{max}}-p_{\texttt{min}}}$ of a pure strategy, ranging from low to high.}
    \label{fig:illustration}
\end{figure*}
\begin{figure}[!tbp]
\centering
    \includegraphics[height=2in]{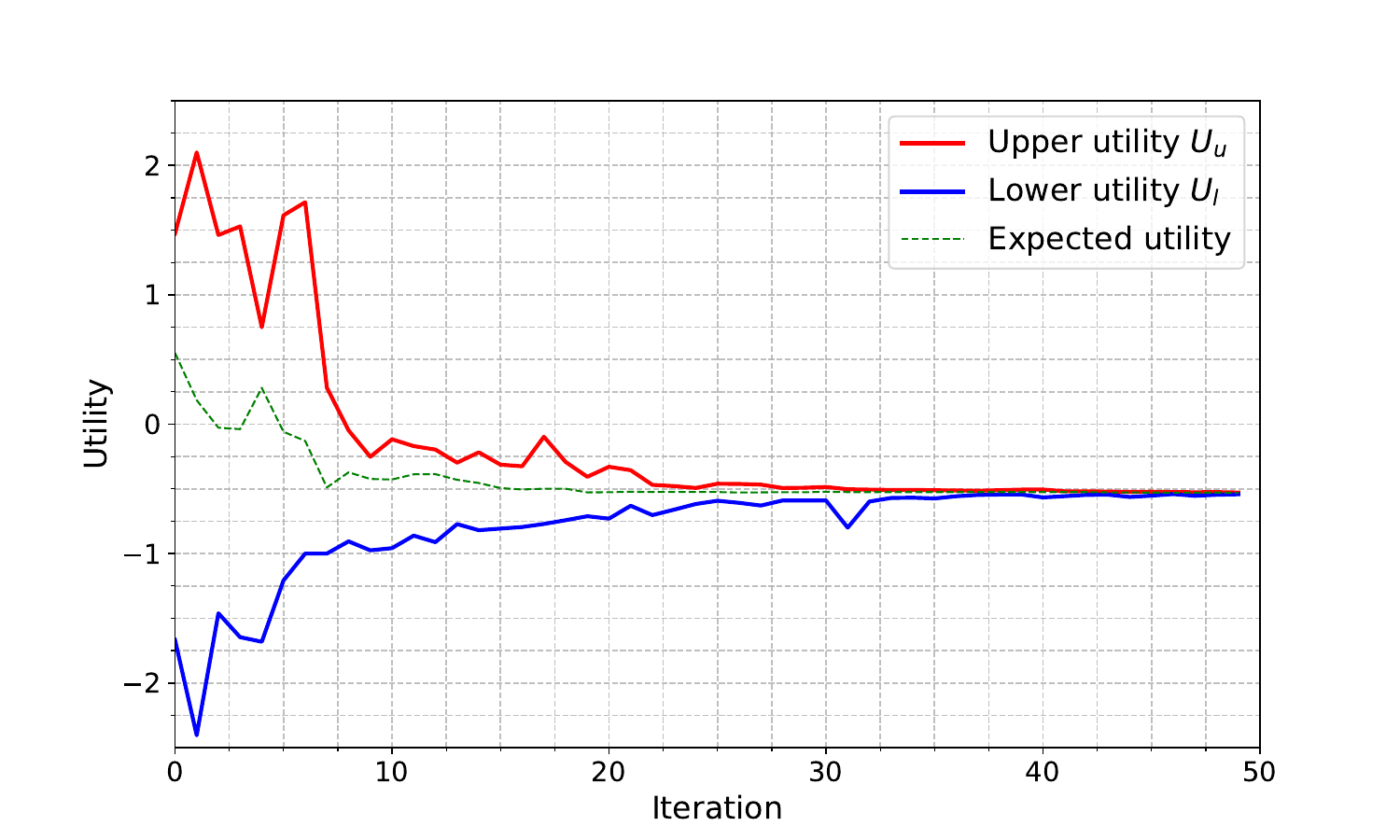}
    \caption{DOA achieves equilibrium (utility -0.53) in CDH robot allocation on $G_2$. The blue curve shows the upper value of the game $U_u$ and the red curve shows the lower value of the game $U_l$ (\algref{alg:doa}). The green curve shows the expected utility of the game in each iteration. Runtime is 3251s.}
    \label{fig:cdh_convergence}
\end{figure}
\begin{figure}[!tbp]
\centering
    \includegraphics[height=2in]{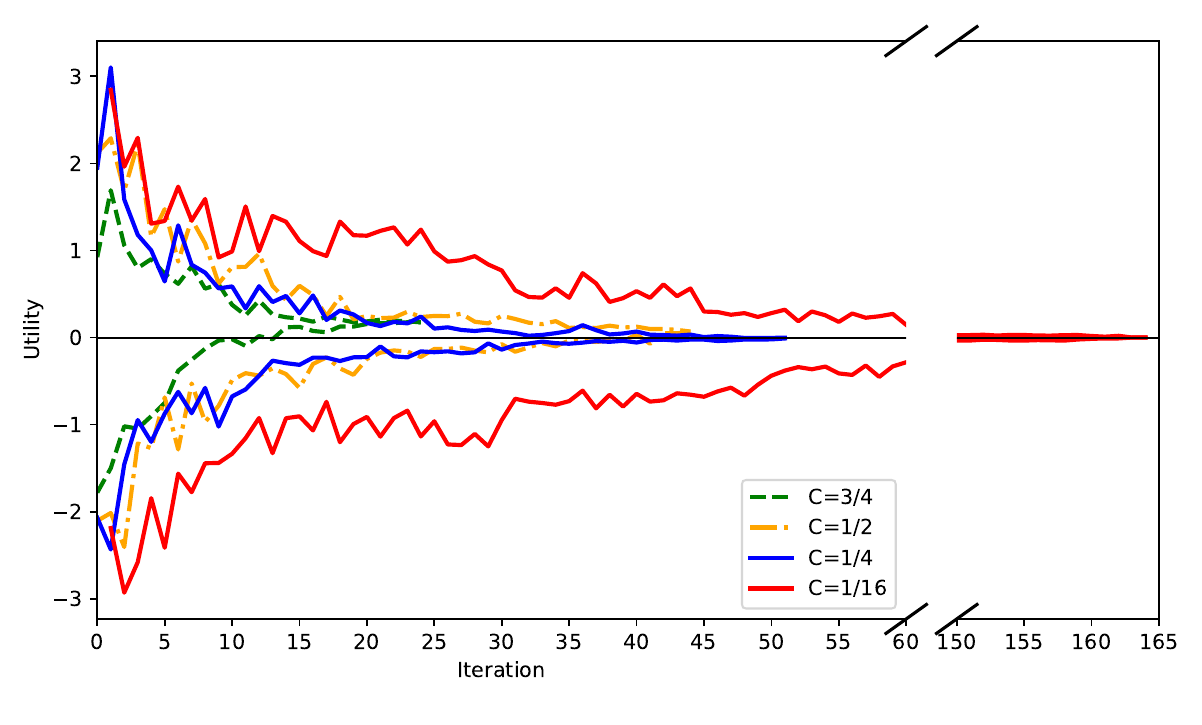}
    \caption{Convergence and runtime with respect to $C$ on $G_1$. When $C=3/4$, runtime is 233s; when $C=1/2$, runtime is 1029s; when $C=1/4$, runtime is 2760s; and when $C=1/16$, runtime is 94173s.}
    \label{fig: converge_comparison}
\end{figure}
\begin{figure*}[!tbp]
    \centering
    \subfloat[On $G_1$]{%
    \label{fig: evaluation_homo_comp}
    \includegraphics[width=0.33\textwidth]{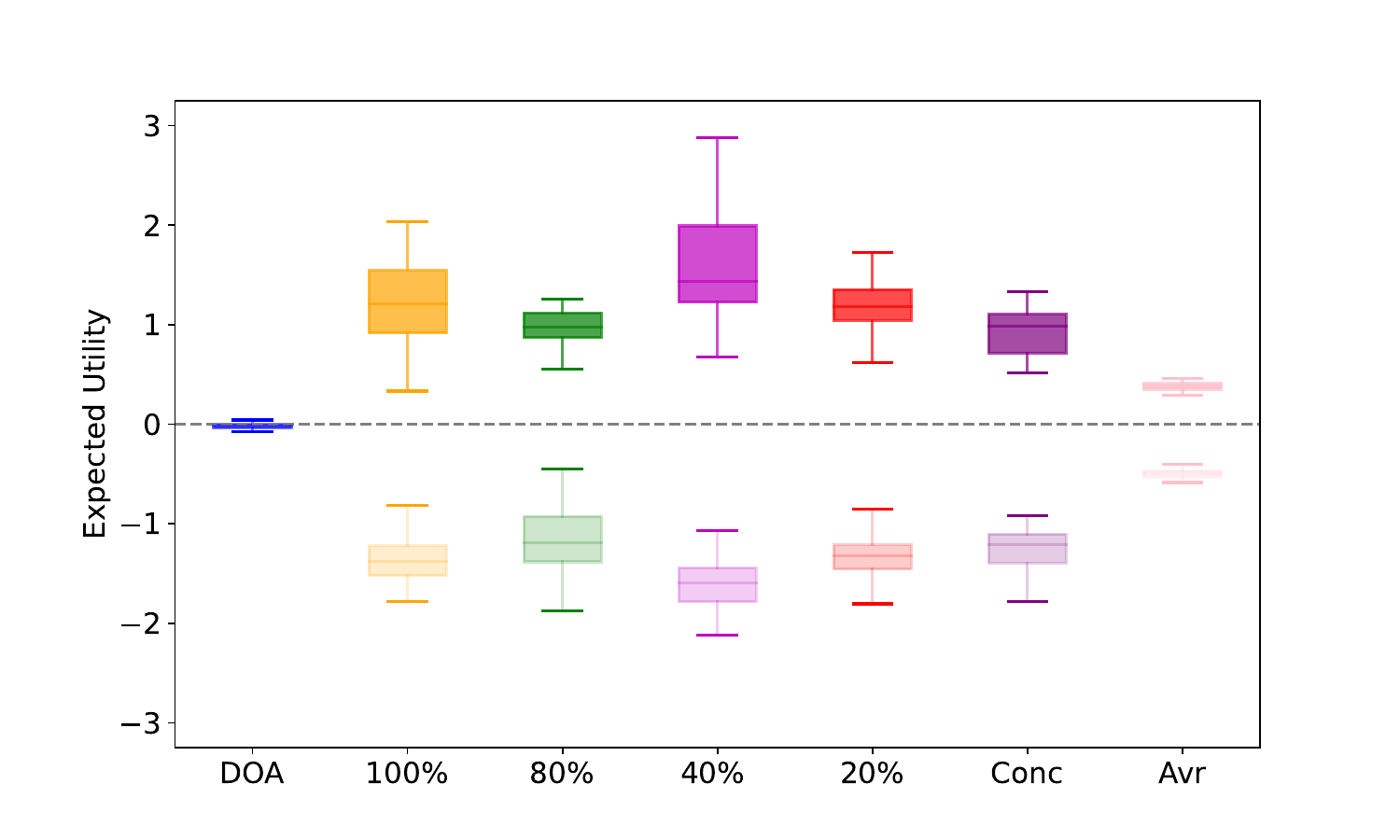}}%
    \hspace{-1em}  
    \subfloat[On $G_2$]{%
    \label{fig: evaluation_heter}
    \includegraphics[width=0.33\textwidth]{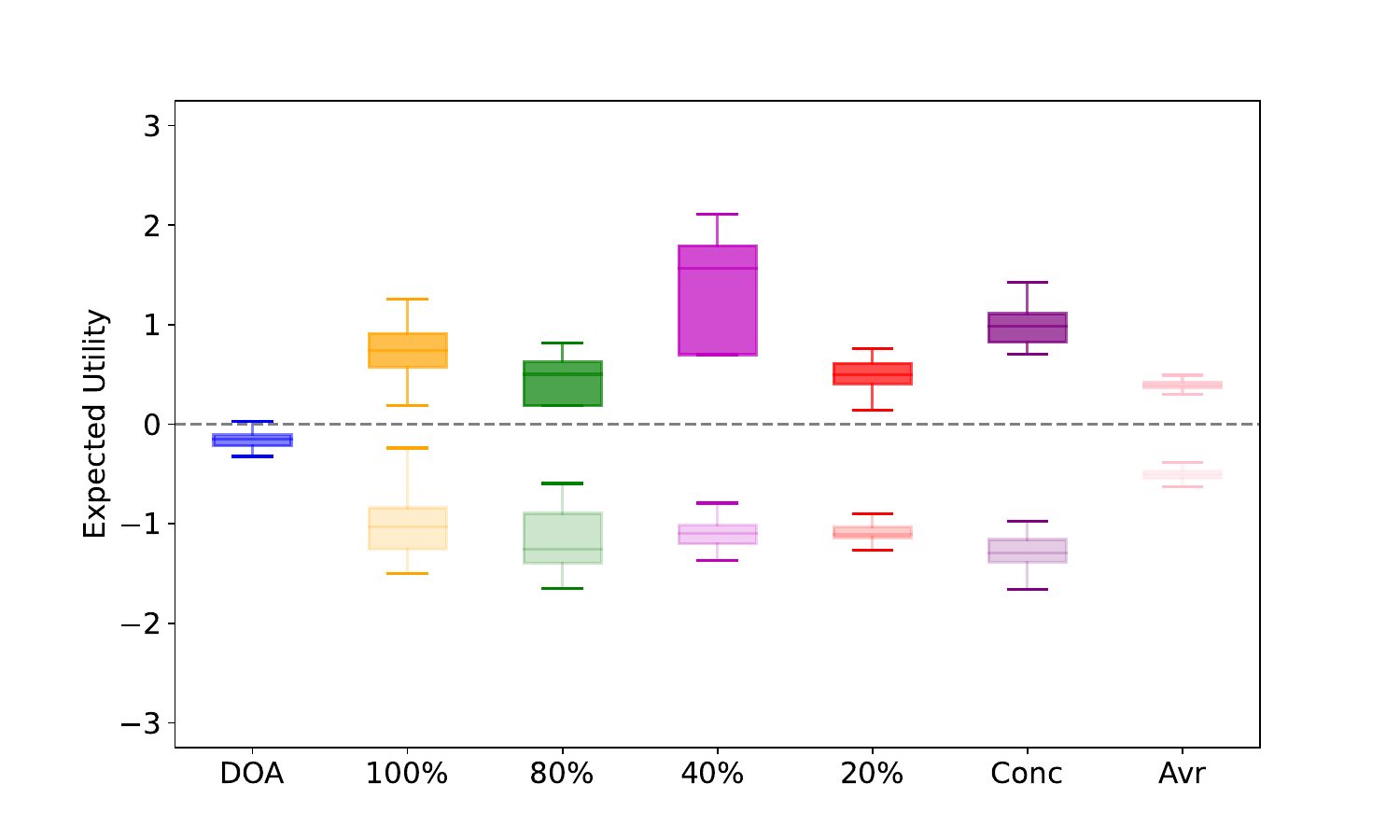}}%
    \hspace{-1em}  
    \subfloat[On $G_3$]{%
    \label{fig: evaluation_homo_5}
    \includegraphics[width=0.33\textwidth]{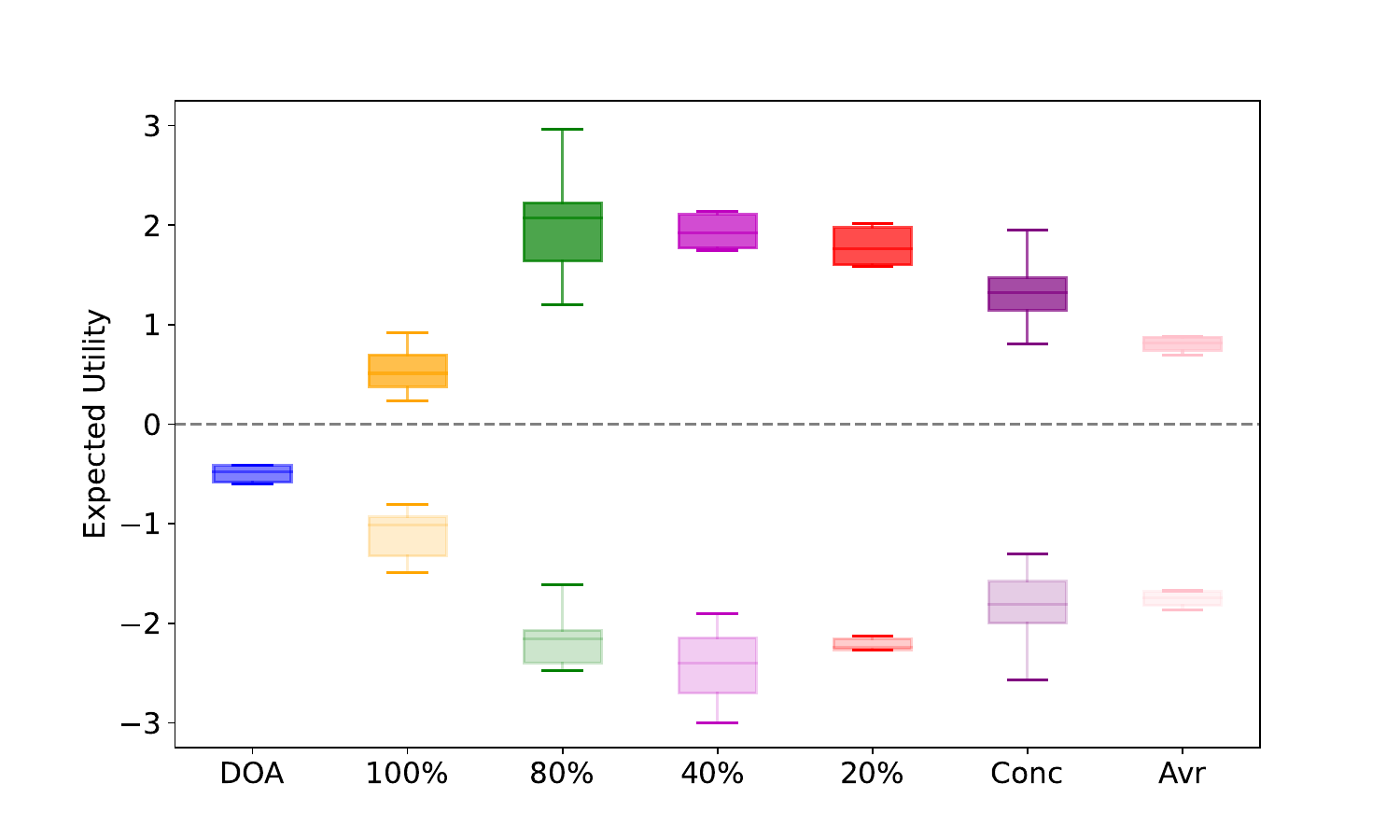}}%
    \caption{Comparison of the expected utilities by DOA and other baselines for CDH robot allocation on graphs (a) $G_1$, (b) $G_2$, and (c) $G_3$.}
    \label{fig:HETER_result}
\end{figure*}

\begin{figure}[!tbp]
\centering
    \label{fig: evaluation_heter_ex}
    \includegraphics[height=2in]{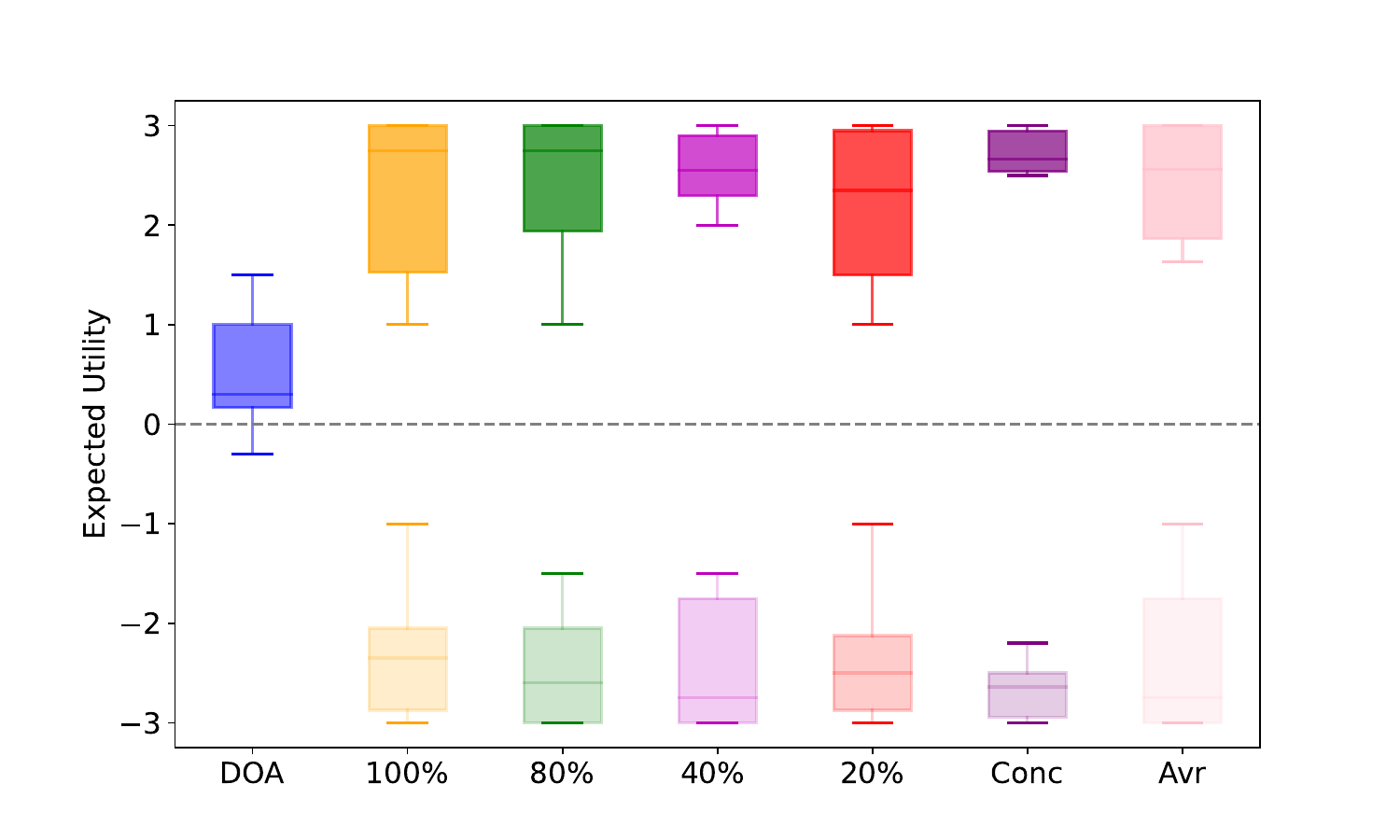}
    \caption{Comparison of the expected utilities by DOA and other baselines for homogeneous robot allocation on $G_2$ with $\mathbf{I}_{12}=\mathbf{I}_{23}=\mathbf{I}_{31}=500$.}
    \label{fig:result}
\end{figure}
In this section, we conduct numerical simulations to demonstrate the effectiveness of DOA in computing the Nash equilibrium for the robot allocation games with homogeneous, linear heterogeneous, and CDH robots. 
Following the settings in~\cite{berman2009optimized}, we use the fraction of the robot population instead of the number of robots to represent the allocation amount. Here, $\mathbf{S}_{ij}$ denotes the fraction of the population of the $i$-th type robots allocated on the $j$-th node. Then we have $\sum_j\mathbf{S}_{ij}=1$ for the $i$-th robot type. The initial robot distributions $\mathbf{d}_x$ and $\mathbf{d}_y$ are also the fractions of the robot population with $\sum_j\mathbf{d}_{x}^i=\sum_j\mathbf{d}_{y}^i=1$.
With homogeneous and linear heterogeneous robots, we consider two players to allocate robots on a three-node directed graph $G_2=(\mathcal{V}_2,\mathcal{E}_2)$ with $\mathcal{V}_2=\{1,2,3\}$ and $\mathcal{E}_2=\{(1,2),(2,3),(3,1)\}$, i.e., $G_2$ in \figref{fig: paper_3cases}. 
For the CDH robot allocation, We evaluate three different cases. In the first case, two players allocate robots on a three-node directed graph $G_1=(\mathcal{V}_1,\mathcal{E}_1)$ with $\mathcal{V}_1=\{1,2,3\}$ and $\mathcal{E}_1=\{(1,2),(1,3),(2,1),(2,3),(3,1),(3,2)\}$, i.e., $G_1$ in \figref{fig: paper_3cases}. In the second case, two players allocate robots on a thee-node directed graph $G_2=(\mathcal{V}_2,\mathcal{E}_2)$ with $\mathcal{V}_2=\{1,2,3\}$ and $\mathcal{E}_2=\{(1,2),(2,3),(3,1)\}$, i.e., $G_2$ in \figref{fig: paper_3cases}. In the third case, two players allocate robots on a five-node directed graph $G_3=(\mathcal{V}_3,\mathcal{E}_3)$ with $\mathcal{V}_3=\{1,2,3,4,5\}$ and $\mathcal{E}_3=\{(1,2),(1,5),(2,3),(2,4),(3,4),(4,3),(4,5),(5,1)\}$, i.e., $G_3$ in \figref{fig: paper_3cases}. 

The initial robot distribution is denoted as $\mathbf{d}_x$ for Player 1 and $\mathbf{d}_y$ for Player 2. The reachable set $\mathcal{X}(\mathbf{d}_x)$ and $\mathcal{Y}(\mathbf{d}_x)$ (light blue region in \figref{fig:illustration}) is calculated based on the extreme action space $\hat{\mathcal{T}}_x$ and $\hat{\mathcal{T}}_y$ as in \eqref{eq:extreme} of \secref{ssec:reachable set}. The elements of extreme action space can be used to form the boundary of the mixed strategies. The number of elements in extreme action space $\hat{\mathcal{T}}_x$ and $\hat{\mathcal{T}}_y$ are denoted as $t_x$ and $t_y$, respectively. Then the best response strategy in the DOA (\algref{alg:doa}, \linref{alg1.2}) for Player 1 can be calculated by:
\begin{align}
\label{eq: maximal_ev}
    \max_{\mathbf{S}_x, \lambda}\sum_{k=1}^K P_{y,k} \sum_{i=1}^M \sum_{j=1}^N  u(\mathbf{S}_{x,ij},\mathbf{S}_{y,ij}^k).
\end{align}
\begin{align*}
    \text{s.t.}\quad
    &\sum_{j=1}^N\mathbf{S}_{x,ij}               =1, i\in\{1,2,3\};\\
    &\sum_{k=1}^{t_x}\lambda_{k}^{i}\mathbf{T}_x^k =\mathbf{S}_{x,i},\enspace i\in\{1,2,3\},\enspace \mathbf{T}_x^k\in \hat{\mathcal{T}}_x;\\
    &\sum_{k=1}^{t_x}\lambda_k^i          =1, i\in\{1,2,3\}.
\end{align*}
We set the outcome threshold $C$ in \eqref{eq:heter-utility} to be $0.25$ for the game with both homogeneous and CDH robots. The intrinsic transformation ratios are $\mathbf{I}_{12}=\mathbf{I}_{23}=\mathbf{I}_{31}=2$. We record 30 trials of the outcome for the following six baselines. 
\begin{itemize}
    \item \emph{Baseline 1:} Change $20\%$ of one player strategy from the mixed strategy $\mathbf{\Delta}_X$ calculated by DOA to the random strategies within constraints. The opponent uses the DOA strategy.
    \item \emph{Baseline 2:} Change $40\%$ of one player strategy from the mixed strategy $\mathbf{\Delta}_X$ calculated by DOA to random strategies within constraints. The opponent uses the DOA strategy.
    \item \emph{Baseline 3:} Change $80\%$ of one player strategy from the mixed strategy $\mathbf{\Delta}_X$ calculated by DOA to random strategies within constraints. The opponent uses the DOA strategy.
    \item \emph{Baseline 4:} Change $100\%$ of one player strategy from the mixed strategy $\mathbf{\Delta}_X$ calculated by DOA to random strategies within constraints. The opponent uses the DOA strategy.
    \item \emph{Baseline 5:} One player takes a pure strategy that arbitrarily selects a vertex (strategy) of the reachable set with $100\%$ probability. The opponent uses the DOA strategy.
    \item \emph{Baseline 6:} One player adopts a mixed strategy that selects the vertices (strategies) of the reachable set with uniform probability. The opponent uses the DOA strategy.
\end{itemize}
Suppose Player 1's strategy after change is $\mathbf{\Delta}_X'$, we denote the best response strategy of Player 2 for $\mathbf{\Delta}_X'$ as $\mathbf{S}_y'\in \delta_y(\mathbf{\Delta}_X')$. According to Lemma~\ref{lemma_2}~\cite{adam2021double}, any change from the equilibrium strategy of one player will benefit her opponent. This means if we change Player 1's strategy, then the expected utility $U(\mathbf{\Delta}_X',\mathbf{S}_y')$ tends to decrease from equilibrium value $U(\mathbf{\Delta}_X^*, \mathbf{\Delta}_Y^*)$, namely $U(\mathbf{\Delta}_X',\mathbf{S}_y')<U(\mathbf{\Delta}_X^*, \mathbf{\Delta}_Y^*)$. Particularly, we test two strategy variation schemes. In the first scheme, we adjust Player 2's strategy according to the baselines and calculate Player 1's best response strategy. The outcome of this scheme is represented by box plots with a darker shade in Figs.~\ref{fig: evaluation_homo}, \ref{fig:HETER_result}, and \ref{fig:result}. Conversely, in the second scheme, we modify Player 1's strategy according to the baselines and calculate Player 2's best response strategy. Its outcome is represented by box plots with a lighter shade in Figs.~\ref{fig: evaluation_homo}, \ref{fig:HETER_result}, and \ref{fig:result}. 
\subsection{Homogeneous or Linear Heterogeneous Robots}
\label{ssec:homo_result}
With homogeneous robots, the utility function $u(\cdot)$ in \eqref{eq: maximal_ev} is the classic function \eqref{eq:utility} introduced in \secref{ssec: equilibrium}. In \cite{adam2021double}, the allocation game with homogeneous robots has been solved without considering graph constraints. For linear heterogeneous robot allocation, the approach is the same except that we first transform different robot types into one type. Therefore, we study these two cases together. 

\figref{fig:homo_convergence} shows the convergence of DOA, where the blue curve represents the upper utility of the game $U_u$ and the red curve represents the lower utility of the game $U_l$ (\algref{alg:doa}). In iteration 20, their values converge. According to \lemmaref{lemma_2}, this demonstrates that the mixed strategies computed by DOA achieve the equilibrium. The run time is 4.243s.

We compare our approach (i.e., DOA) with six baselines across ten trials where we randomly generate the initial allocations for the two players on $G_2$. Within each trial, we run baselines with randomness involved 1000 times and illustrate them via the box plot. The results are shown \figref{fig: evaluation_homo}.  
Recall that \lemmaref{lemma_2} states that, in the context of an equilibrium mixed strategy set $(\mathbf{\Delta}_X^*, \mathbf{\Delta}_Y^*)$, any variation in one party's strategy invariably shifts the outcome towards a direction more advantageous to the opposing party. 
Here, the utility function is $u(\mathbf{S}_x-\mathbf{S}_y)$ (\secref{ssec: equilibrium}). Thus, a larger $u$ benefits Player 1 while a small $u$ benefits Player 2. \figref{fig: evaluation_homo} shows the expected values via the first scheme (i.e., darker shades) are always greater than the value obtained by DOA. That is because, with the first scheme, Player 2's strategy is randomly changed (with different levels for different baselines). This change benefits Player 1, as reflected by the increase in the game's expected value. Analogously, the expected values by the second scheme two (i.e., lighter shades) are always less than the value obtained by DOA. Because, with the second scheme, Player 1's strategy is randomly changed, which benefits Player 2, as reflected by the decrease in the game's expected value. In other words, any variation from DOA leads to a worse utility of a corresponding player. Therefore, this further demonstrates that the mixed strategies calculated by DOA achieve the equilibrium of the game with homogeneous or linear heterogeneous robots.  
\subsection{CDH Robots}
\label{ssec:heter_result}
For CDH robot allocation on three-node graphs $G_1$ and $G_2$, we set the initial robot distribution for Player 1 $\mathbf{d}_x$ and for Player 2 $\mathbf{d}_y$ as: 
\begin{equation*}
    \mathbf{d}_x = \begin{bmatrix}
        0.7 & 0.1 & 0.2\\
        0.4 & 0.4 & 0.2\\
        0.3 & 0.1 & 0.6
    \end{bmatrix},\quad 
    \mathbf{d}_y = \begin{bmatrix}
        0.2 & 0.2 & 0.6\\
        0.35 & 0.15 & 0.5\\
        0.4 & 0.2 & 0.4
    \end{bmatrix}. 
\end{equation*}
For the five-node graph $G_3$, we set the initial robot distributions as:
\begin{align*}
    \mathbf{d}_x &= \begin{bmatrix}
        0.2& 0.3& 0.1& 0.1& 0.3\\
        0.3& 0.1& 0.4& 0.1& 0.1\\
        0.2& 0.1& 0.1& 0.1& 0.5
    \end{bmatrix},\\
    \mathbf{d}_y &= \begin{bmatrix}
        0.1& 0.2& 0.3& 0.1,& 0.3\\
        0.35& 0.15& 0.1& 0.1& 0.3\\
        0.15& 0.2& 0.35& 0.1& 0.2
    \end{bmatrix}. 
\end{align*}
The column of $\mathbf{d}_x$ refers to different nodes and the row of $\mathbf{d}_x$ refers to different robot types. 
The utility function is $u_{\texttt{CDH}}(\cdot)$ in \eqref{eq:heter-utility}. 

\figref{fig:illustration} gives an example of the mixed strategies obtained by DOA. It is observed that for both players the mixed strategies calculated by DOA lie in the reachable set, which illustrates the graph constraints. In \figref{fig:cdh_convergence} the upper utility of the game $U_u$ (blue curve) and the lower utility of the game $U_l$ (red curve) converge to the expected utility of the game after after 40 iterations. According to \lemmaref{lemma_2}, this demonstrates that the mixed strategies computed by DOA are the equilibrium (or optimal) strategies. The DOA run time for $G_2$ is 3251s, and the DOA run time for $G_3$ is 11401s.

In \figref{fig: converge_comparison}, we calculate the equilibrium on $G_1$, a fully connected 3-node graph to demonstrate the effect of $C$. Here, the equilibrium value of the game should be zero, as it is a symmetric game where neither player should have an advantage over the other, as shown in \cite{roberson2006colonel}. We demonstrated the effect of $C$ by its four values, $C \in \{3/4, 1/2, 1/4, 1/16\}$. When $C=3/4$, the equilibrium value deviates from zero, suggesting it might be a local optimum. When $C$ is reduced to $1/4$, the equilibrium value converges to zero but takes a significantly longer time. This indicates that a larger $C$ may affect convergence accuracy to equilibrium and a smaller $C$ prolongs the convergence time. In other words, $C$ plays a trade-off between efficiency and accuracy. Moreover, the critical $C$ value at which DOA converges to zero in this case is approximately $1/2$. When $C$ exceeds $1/2$, the DOA may not converge to zero. Additionally, when $C$ is very small (e.g., $C=1/16$), the convergence takes a very long time and consumes a large amount of memory.
In \figref{fig:HETER_result}, we compare the expected utilities by DOA and other baselines (formed via the two schemes) for CDH robot allocation on different graphs $G_1$, $G_2$ and $G_3$. Similar to the results of homogeneous robot allocation, the expected utility $U(\mathbf{\Delta}_X^*, \mathbf{\Delta}_Y^*)$ is nearly zero and any change from the DOA strategies of one player always benefits her opponent, which is in line with the properties of equilibrium (\lemmaref{lemma_2}). 
In \figref{fig:result}, we set the intrinsic transformation ratio as a large number, i.e., $\mathbf{I}_{12}=\mathbf{I}_{23}=\mathbf{I}_{31}=500$ to model the case of the \textit{absolute dominance} heterogeneous robot allocation. The result also aligns with \lemmaref{lemma_2}. Therefore, \figref{fig:HETER_result} and \figref{fig:result} further verify that the mixed strategies calculated by DOA are the equilibrium strategies of the game with the CDH robots.  

In \figref{fig: evaluation_homo}, \figref{fig:HETER_result}-(b), (c), and \figref{fig:result}, we observe that the value of expected utility calculated by DOA is not zero. Notably, the conventional equi-resource Colonel Blotto game is a zero-sum game \cite{behnezhad2017faster}, \cite{roberson2006colonel} and the utility at the equilibrium is zero. This can be demonstrated by \figref{fig:HETER_result}-(a) where the two players have the same number of robots and the robots can move flexibly with $G_1$ a complete graph. However, due to the graph constraints in \figref{fig: evaluation_homo}, \figref{fig:HETER_result}-(b), (c), and \figref{fig:result}, the game between the two players is not symmetric with different initial robot distributions. This non-symmetry leads to a non-zero utility value at the equilibrium.

\section{Conclusions and Future Work}
\label{sec: conclusion}
In this paper, we formulated a two-player robot allocation game on graphs. Then we leveraged the DOA to calculate the equilibrium of the games with homogeneous, linear heterogeneous, and CDH robots. Particularly, for CDH robot allocation, we designed a new transformation approach that offers reasonable comparisons between different robot types. Based on that, we designed a novel utility function to quantify the outcome of the game. Finally, we conducted extensive simulations to demonstrate the effectiveness of DOA in finding the Nash equilibrium.    

Our first future work is to relax the constraints of \textit{Assumption 3} and discuss other fixed battle orders, for example $R_1$ first fights $R_2$, then $R_2$ fights $R_3$, and finally $R_3$ fights $R_1$. In this scenario, the structure of the outcome interface $\pi_{\texttt{oi}}$ would differ, necessitating the formulation of a corresponding utility function for optimization in DOA.
Second, we will include the spatial and temporal factors in the robot allocation problem. The current setting assumes instantaneous transitions of robots between nodes. 
However, in realistic scenarios, spatial and temporal factors such as the distances between nodes and varying transition speeds of the robots need to be considered. 
Moreover, we will incorporate the elements of deception into the game, which differs from the current setting that assumes both players to have real-time awareness of each other's strategy and allocation. This would encourage the players to strategically mislead opponents by sacrificing some immediate gains to achieve larger benefits in the longer term. 

\bibliographystyle{IEEEtran}
\bibliography{ref}
\appendix
We use an example to explain the hardness
caused by increasing the number of robot types to four, i.e., $M=4$. Suppose there are four types of robots $R_1, R_2, R_3, R_4$, and the intrinsic matrix is 
$$
\mathbf{I}=\begin{bmatrix}
    1&2&-&1/2\\
    1/2&1&2&-\\
    -&1/2&1&2\\
    2&-&1/2&1.
\end{bmatrix}
$$
Note that the entries $\mathbf{I}_{13},\mathbf{I}_{24},\mathbf{I}_{31},\mathbf{I}_{42}$ (denoted as $-$ in the intrinsic matrix) are not defined, as they are not involved in the calculation. For instance, on a node, Player 1's allocation is $(u_1,v_1,w_1,z_1)=(1,0,1,0)$ while Player 2's allocation is $(u_2,v_2,w_2,z_2)=(0,1,0,1)$. Therefore the remaining robot distribution after subtraction is $(\delta u,\delta v,\delta w,\delta z)=(1,-1,1,-1)$. This scenario presents a challenging outcome to determine a clear winner. Consider the intrinsic relationship between $R_1$ and $R_2$ and that between $R_3$ and $R_4$. Given that $R_1$ equals $2R_2$, the elimination of $1R_1$ from $1R_2$ results in a remaining balance of $0.5R_1$. Similarly, the elimination of $1R_3$ from $1R_4$ leads to a balance of $0.5R_3$. From the perspective of the remaining robots, represented as $(0.5,0,0.5,0)$, it shows that Player 1 wins. However, if we instead consider the countering relationship between $R_2$ and $R_3$ and that between $R_4$ and $R_1$. Following the same logic, the outcome becomes $(0,-0.5,0,-0.5)$, showing Player 2 is the winner. It is important to notice that both approaches adhere to the \emph{elimination transformation} rule in \secref{ssec:NLH-heter}. This dichotomy shows the complexity and ambiguity in adjudicating a definitive winner in this context. This implies that in the case of $M>3$, it is necessary to introduce new mechanisms for determining win or loss on each node.

Notably, the intrinsic matrix in the above example is not deliberately constructed. In fact, any intrinsic matrix that reflects the CDH relationships between different robot types will inevitably lead to such issues. In \secref{sssec:u_for_heter}, we prove that for $M=3$, outcome interface $\pi_{oi}(\mathbf{x})=0$ is the demarcation surface of win and loss. However, the winning condition described in \secref{ssec:NLH-heter} does not hold for $M>3$, and thus the demarcation surface does not hold either. We conjecture that the demarcation for $M>3$ is a polytope, which we leave for future study. 
\end{document}